\documentclass[11pt]{article}

\usepackage{mathtools}
\usepackage[margin=1in]{geometry}
\usepackage{hyperref}
\usepackage{ifthen}
\usepackage{amsmath,amssymb,amsthm}
\usepackage{bbm,dsfont}
\usepackage{censor}

\usepackage{xcolor}
\usepackage{placeins}
\usepackage{caption}
\usepackage{tikz}
\usetikzlibrary{positioning,cd}

\definecolor{UCnavyy}{rgb}{0.094117, 0.168627, 0.2862745}           
\definecolor{navy}{rgb}{0, 0.415686, 0.588235}  

\definecolor{cf9f9f9}{RGB}{249,249,249}
\definecolor{cb3b3b3}{RGB}{179,179,179}
\definecolor{c808080}{RGB}{128,128,128}
\definecolor{c1a1a1a}{RGB}{26,26,26}
\definecolor{cffffff}{RGB}{255,255,255}

\usepackage{flafter}
\usepackage{comment}
\usepackage{cleveref}

\usepackage{natbib}
\bibpunct{(}{)}{;}{a}{,}{,}

\usepackage{apptools}
\AtAppendix{\counterwithin{theorem}{section}}

\usepackage{thmtools,thm-restate}
\newtheorem{theorem}{Theorem}[section]

\newtheorem{lemma}[theorem]{Lemma}

\newtheorem{proposition}[theorem]{Proposition}
\newtheorem{corollary}[theorem]{Corollary}
\theoremstyle{definition}
\newtheorem{definition}[theorem]{Definition}
\newtheorem{remark}[theorem]{Remark}
\newtheorem{assumption}{Assumption}

\newtheorem*{informaltheorem*}{Theorem}
\newtheorem{manualtheoreminner}{Theorem}
\newenvironment{manualtheorem}[1]{%
  \renewcommand\themanualtheoreminner{#1}%
  \manualtheoreminner
}{\endmanualtheoreminner}

\usepackage[named]{algo}

\renewcommand{\emptyset}{\varnothing}
\renewcommand{\epsilon}{\varepsilon}
\DeclareMathOperator*{\E}{\mathbb{E}}

\DeclareMathOperator*{\argmin}{\mathrm{arg\,min}}

\newcommand{\ind}{\mathds{1}}

\usepackage{afterpage}
\usepackage{inconsolata}
\usepackage{enumitem}

\hypersetup
{
    pdfauthor={Sanjoy Dasgupta, Gaurav Mahajan, Geelon So},
    pdfsubject={Convergence of online k-means},
    pdftitle={Convergence of online k-means}
}

\title{Convergence of online \textit{k}-means}
\author{
Sanjoy Dasgupta\thanks{Dept.\@ of Computer Science and Engineering, UC San Diego. \{\href{mailto:dasgupta@eng.ucsd.edu}{\texttt{dasgupta}}, \href{mailto:gmahajan@eng.ucsd.edu}{\texttt{gmahajan}}, \href{mailto:agso@eng.ucsd.edu}{\texttt{agso}}\}\texttt{@eng.ucsd.edu}.} \and
Gaurav Mahajan$^*$ \and
Geelon So$^*$
}
\date{\today}

\begin{document}

\maketitle

\begin{abstract}
We prove asymptotic convergence for a general class of $k$-means algorithms performed over streaming data from a distribution---the centers asymptotically converge to the set of stationary points of the $k$-means cost function. To do so, we show that online $k$-means over a distribution can be interpreted as stochastic gradient descent with a stochastic learning rate schedule. Then, we prove convergence by extending techniques used in optimization literature to handle settings where center-specific learning rates may depend on the past trajectory of the centers.
\end{abstract}
\newpage
\tableofcontents
\newpage

\section{Introduction}
\label{sec:motivation}
Lloyd's method \citep{lloyd1982least} is a popular iterative procedure for $k$-means clustering a finite dataset in $\mathbb{R}^d$. At each step, the algorithm proposes $k$ centers, say $W_1,\ldots, W_k \in \mathbb{R}^d$. Each data point is then mapped to its closest center, partitioning the dataset into $k$ clusters. The update simply sets each center to the mean of its corresponding cluster data. Since each step requires a pass over the whole dataset, large-scale data and streaming settings often use online variants of $k$-means, computing updates on single data points or mini-batches of data points.

Consider \emph{online \textit{k}-means algorithms} with updates that (i) receive a data point $X$, (ii) find the closest center $W_{i}$ among $W_1,\ldots, W_k$, and (iii) update $W_{i}$ using $X$. The long-term behavior of this procedure is unknown when applied to a never-ending stream of data points that is drawn from an underlying data distribution $p$ on $\mathbb{R}^d$. This leads to the following question:
\begin{quotation}
\noindent \emph{If $X^{(1)}, X^{(2)},\ldots$ come from an underlying data distribution $p$, do these forms of online $k$-means algorithms converge to local optima of the $k$-means cost function $f$ on $p$?}
\end{quotation}

\paragraph{A motivating example for analysis}
\cite{bottou1995convergence} define an online $k$-means algorithm used in practice, which we call the \ref{alg:naive} algorithm. For each $i = 1,\ldots, k$, it simply sets the center $W_i$ to the mean of all its previous updates, which can be computed in a streaming fashion. It does so by maintaining a counter $N_i$ for the number of times each center has been updated so far. If $W_{i}$ is the center closest to the next data point $X$, the update is:
\begin{align*} 
W_{i} &\leftarrow W_{i} - \frac{1}{N_{i} + 1} \big(W_{i} - X\big) \quad\textrm{and}\quad
N_{i} \leftarrow N_{i} + 1.
\end{align*}
\begin{figure*}[h!]
\begin{center}
\begin{minipage}{0.9\textwidth}
\hrulefill
\begin{algorithm}{online Lloyd's}{
    \label{alg:naive}
    \qinit{$k$ arbitrary distinct centers $W \in \mathbb{R}^{k\times d}$ from the support of $p\vphantom{\big|_p}$}}
    \qfor $n = 0,1,2,\ldots\vphantom{\displaystyle\frac{1}{N}}$ \\
        sample data point $X\sim p\vphantom{\Big|}$\\
        identify closest center $i \leftarrow \argmin_{j \in [k]}\, \|W_j - X\|\vphantom{\Big|}$\\
    update counter $\displaystyle N_i \leftarrow N_i +1 \vphantom{\Big|}$   \\ 
    update center $\displaystyle W_i \leftarrow W_i - \frac{1}{N_i}\cdot \big(W_i - X\big) $   
    \qrof
\end{algorithm}
\hrulefill
\end{minipage}
\vspace{5pt}

\begin{minipage}{0.85\textwidth}
\small{\textsc{algorithm}. A simple online $k$-means algorithm introduced by \cite{bottou1995convergence}. At any point in time, each center $W_i$ is the mean of all its previous updates.} 
\end{minipage}
\end{center}
\end{figure*}

\noindent Generalizing, we consider a broader class of \ref{alg:distributional-kmeans} algorithms whose update has the form:
\[W_{i} \leftarrow W_{i} - H_{i} \cdot \big(W_{i} - X\big).\]
Here, $X$ is a random draw from $p$ and $i$ is the index of the closest center $W_{i}$ to $X$. Further, $H_{i} \in [0,1]$ is a center-specific stochastic learning rate that may depend arbitrarily on the past. This yields a simple geometric meaning to the $W_i$'s: each is a convex combination of all its previous updates.
\begin{figure*}
\begin{center}
\begin{minipage}{0.9\textwidth}
\hrulefill
\begin{algorithm}{online $k$-means}{
    \label{alg:distributional-kmeans}
    \qinit{$k$ arbitrary distinct centers $W \in \mathbb{R}^{k\times d}$ from the support of $p\vphantom{\big|_p}$}}
    \qfor iteration $n = 0,1,2,\ldots\vphantom{\big|_{p_i}}$\\ 
    \qdo
        sample data point $X\sim p\vphantom{\big|_{p_i}}$ \\
        identify closest center $i \leftarrow \argmin_{j \in [k]}\, \|W_j - X\|\vphantom{\big|_{p_i}}$\\
    update closest center $\displaystyle W_i \leftarrow W_i - H_i\cdot \big(W_i - X\big) \vphantom{\big|_{V_{i}(W^{(n)})}}$   
    \qrof
\end{algorithm}
\hrulefill
\end{minipage}
\vspace{5pt}

\begin{minipage}{0.85\textwidth}
\small{\textsc{algorithm}. A class of online $k$-means algorithms. Here, $W = (W_1,\ldots, W_k)$ is the tuple of $k$ centers maintained by the algorithm and $H_i \in [0,1]$ is a (stochastic) learning rate for the $i$th center. $X$ is a random sample from data distribution $p$.} 
\end{minipage}
\end{center}
\end{figure*}

\paragraph{Challenges to analysis} Despite its algorithmic simplicity, \ref{alg:distributional-kmeans} has eluded analysis. While \mbox{$k$-means} is often analyzed by recasting it as stochastic gradient descent (SGD), this is a setting for which existing optimization literature is insufficient. The difficulty is that centers can learn at different rates that possibly depend on the whole history of the algorithm. To circumvent the issue, previous work (e.g.\@ \cite{tang2017convergence}) replace the center-specific learning rate of $\frac{1}{N_i + 1}$ by a uniform-across-centers and deterministic learning rate, say $\frac{1}{n}$, where $n$ is the number iterations that has elapsed in the algorithm.

\subsection{Main contributions}
We prove that a large class of online $k$-means algorithms asymptotically converge under reasonable assumptions to the set of stationary points of the $k$-means objective. In particular, we show:
\paragraph{Connection to stochastic gradient descent} We prove in \Cref{lem:k-means-grad} that algorithms in this family perform SGD on the $k$-means cost. While known for $k$-means over finite datasets \citep{bottou1995convergence}, the result does not trivially extend to distributions---the essential difference is that there are finitely many ways to cluster a finite dataset, but infinitely many ways to cluster $\mathbb{R}^d$.

\paragraph{Convergence of online $k$-means algorithms} Standard techniques from optimization literature are able to analyze SGD with uniform learning rates, but they are unable to handle the variant of SGD performed by \ref{alg:distributional-kmeans}, which has center-specific learning rates. To show convergence, we extend the techniques from \cite{bertsekas2000gradient} to cover non-uniform learning rates. 

Of course, not every choice of non-uniform learning rates $H_i$ will lead to convergence. As an extreme example, if an adversary can set the learning rate of a center $W_i$ to zero, then the iterates will never converge to a stationary point. To prove convergence, we need to impose additional conditions. The key property that we shall require for convergence is that if a center $W_i$ is far from its cluster mean---the mean of its Voronoi cell---then with constant probability, it is updated at a rate not too much slower than the rest of the centers. \Cref{thm:iterates-conv-simple} proves convergence.

\paragraph{Convergence of a generalized online Lloyd's algorithm} It turns out that online Lloyd's is particularly difficult to analyze. It is poorly conditioned in the sense that nothing seems to prevent iterates from making rare but large jumps---it is unclear whether online Lloyd's algorithm satisfies the assumptions from our convergence theorem, \Cref{thm:iterates-conv-simple}. 

While the \ref{alg:naive} algorithm falls into the family of \ref{alg:distributional-kmeans} algorithms we consider, it turns out that it is particularly difficult to analyze. It is poorly conditioned in the sense that nothing seems to prevent iterates from making rare but large jumps---it is unclear whether \ref{alg:naive} satisfies the assumptions from our convergence theorem. Furthermore, \ref{alg:naive} may differ significantly from the original offline Lloyd's algorithm. In the original, centers are updated to the mean of the current clusters. But in online Lloyd's, centers are set to the mean of all previous updates. But this mean-of-all-previous-updates does not generally well-approximate the mean of the current cluster because the underlying clusters drift about throughout the whole algorithm. 

Instead, to design an online version of Lloyd's algorithm with asymptotic guarantees, we start from the interpretation of Lloyd's algorithm as preconditioned gradient descent. Then, we define a \ref{alg:onlinelloyd} algorithm as its stochastic analog, which concurrently keeps an estimate of the preconditioner. We prove the consistency of our estimator to the Lloyd preconditioner in \Cref{sec:lloyd}, lending our algorithm the interpretation of a natural extension of Lloyd's. 

Additionally, we prove that \ref{alg:onlinelloyd} also achieves asymptotic convergence. To state the result, we say that a $k$-tuple of centers $w \in \mathbb{R}^{k\times d}$ is \emph{degenerate} if at least two of the centers coincide, $w_i = w_j$ for some $i \ne j$. The following is an informal restatement of \Cref{thm:lloydconvergence}.

\begin{informaltheorem*}[informal]
    Let $p$ be a continuous density with bounded support on $\mathbb{R}^d$ and let $f$ be its \mbox{$k$-means} cost. Suppose that the set of stationary points $\{\nabla f = 0\}$ has no degenerate limit points. Let $(W^{(n)})_{n=0}^\infty$ be the iterates  of the \ref{alg:onlinelloyd} algorithm. Then, the iterates asymptotically converges to the set of stationary points:
    \[\limsup_{n \to \infty} \, \inf_{w \in \{\nabla f = 0\}}\, \|W^{(n)} - w\| = 0.\]
\end{informaltheorem*}

\subsection{Related work}

An essential goal of unsupervised learning is to simplify the signal from data, while preserving meaning relevant for downstream tasks. In \textit{$k$-means clustering} or \textit{vector quantization}, this simplification is achieved by discretizing the data space $\mathbb{R}^d$ into a finite set of prototypes $w_1,\ldots, w_k \in \mathbb{R}^d$. Any data point can then be clustered with/approximated by the nearest $w_i$. Thus, given a data distribution $p$ on $\mathbb{R}^d$, it is natural to aim to find a discretization $w = (w_1,\ldots, w_k)$ that minimizes the average $\ell_2^2$-reconstruction error:
\[\E_{X \sim p}\left[\min_{i \in [k]}\, \|w_i - X\|^2\right] = \int \min_{i \in [k]}\, \|w_i - x\|^2 \, p(x)\, dx.\]

But since we do not generally access $p$ directly but through random samples---its empirical measures---this raises the statistical question of how much data is theoretically required to estimate an optimal clustering. To this end, \cite{pollard1981strong} shows under weak assumptions that the optimal clustering of empirical measures converge almost surely to the optimal clustering of $p$. This was also followed by much work in the clustering and vector quantization community showing rates of convergence; see \cite{bachem2017uniform} and their related works section.

However, recovering an (arbitrarily) optimal clustering of an empirical measure quickly becomes computationally infeasible as the size of the dataset grows \citep{aloise2009np,awasthi2015hardness}, so in practice, simple heuristics such as Lloyd's algorithm are used to find local optima \citep{lloyd1982least}. Regarding such algorithms, Pollard remarks: ``I do not know whether the techniques to be developed in this paper can be applied to prove consistency results for [existing efficient algorithms that find] locally optimal partitions.'' While we do not develop on his technique, we show that \ref{alg:distributional-kmeans} asymptotically converges to stationary points of the reconstruction error---equivalently, the $k$-means cost function up to a constant factor---almost surely.

To analyze online $k$-means on finite datasets, \cite{bottou1995convergence} reinterpret the update as gradient descent, which \cite{tang2017convergence} use to prove convergence given uniform and deterministic learning rates, attaining rates of convergence. We consider the setting of online $k$-means over a data distribution with non-uniform and stochastic learning rates. We also show convergence but leave open the question of  rates: one challenge that immediately arises is there may be uncountably many stationary points in the distributional setting. In contrast, the set of stationary points in the finite setting is also finite---hence isolated. 

To analyze SGD, we use standard frameworks to prove convergence \citep{bertsekas2000gradient,li2019convergence}. However, much of the general theory covers only uniform learning rates. Our work introduces a technique that may be applied to prove convergence for more general SGD-based algorithms with non-uniform and stochastic learning rates.

In our analysis of the $k$-means cost, we show that it admits a family of tangent quadratic upper bounds (\Cref{sec:prop}). Thus, $k$-means over distributions, as in the finite setting, fits into the  \emph{majorization-minimization} (MM) scheme (see \cite{mairal2015incremental}). It would be of interest to generalize our work to iterative or online MM algorithms \citep{cappe2009line, karimi2019global}.


\section{Preliminaries} \label{sec:prelim}
Let $p$ be a density on $\mathbb{R}^d$ with bounded second moment. Notice that because $p$ is a density, any Lebesgue measure zero set also has zero probability mass. We denote a tuple of $k$ centers or prototypes in $\mathbb{R}^d$ by $w = \big(w_1,\ldots, w_k\big) \in \mathbb{R}^{k \times d}$. Define the following:
\begin{itemize}
    \item $V(w) = \big(V_1(w),\ldots, V_k(w)\big)$ is the induced Voronoi partitioning\footnote{Strictly speaking, $V(w)$ does not partition $\mathbb{R}^{k \times d}$ because adjacent partitions $V_i(w)$ and $V_j(w)$ share boundary points. However, boundary points form a measure zero set, so we will encounter no problems.}
    of $\mathbb{R}^d$ by $w$
    \[V_i(w) = \bigg\{x \in \mathbb{R}^d : w_i \in \argmin_{w_j} \,\|w_j - x\|\bigg\}\]
    \item $P(w) = \big(P_1(w),\ldots, P_k(w)\big)$ is the probability mass of each of the Voronoi partitions
    \[P_i(w) = \int_{V_i(w)} p(x)\,dx\]
    \item $M(w) = \big(M_1(w),\ldots, M_k(w)\big)$ is the mean/center of mass of each of the Voronoi partitions
    \[M_i(w) = \frac{1}{P_i(w)} \int_{V_i(w)} x\, p(x)\,dx.\]
\end{itemize}
While these functions are defined for all $w \in \mathbb{R}^{k \times d}$, we will be able to restrict our analysis to the set of non-degenerate tuples, where none of the centers coincides:
\[\mathcal{D} := \{w \in \mathbb{R}^{k \times d} : w_i \ne w_j, \ \forall i\ne j\}.\]
Later on, we will restrict the support of $p$ to a closed ball $B(0,R)$ centered at the origin of radius $R$ in $\mathbb{R}^d$. Let $\mathcal{D}_R$ be the set of non-degenerate tuples in $B(0,R)$:
\[\mathcal{D}_R := \{w \in \mathcal{D} : w_i \in B(0,R), \ \forall i \in [k]\}.\]
Given a Borel set $S \subset \mathbb{R}^d$ with positive probability mass $p(S) > 0$, let $p\big|_S$ denote the distribution obtained by restricting $p$ onto $S$. Finally, we will also let $[k]$ denote the set $\{1,\ldots, k\}$.

\subsection{The \textit{k}-means problem}
The $k$-means objective is to minimize:
\begin{equation}\label{eqn:kmeans-cost}
f(w) := \frac{1}{2}\sum_{i \in [k]}\ \int_{V_i(w)} \| w_i - x\|^2\, p(x) \,dx.
\end{equation} 
While the objective is non-convex, we show that it is smooth on $\mathcal{D}$. Thus, we aim for convergence to stationary points---the iterates $W^{(n)}$ approach the set of stationary points $\{\nabla f = 0\}$ as $n \to \infty$. 

\begin{definition}[Asymptotic convergence]
Let $\mathcal{D}$ be a domain and $f : \mathcal{D} \to \mathbb{R}$ differentiable. We say that a sequence of points $(w^{(n)})_{n=0}^\infty$ in $\mathcal{D}$ \emph{asymptotically converges to stationary points of $f$} if all limit points of $(w^{(n)})_{n=0}^\infty$ are stationary points of $f$,
\[\bigcap_{n \geq 0} \overline{(w^{(n')})_{n' \geq n}} \subset \{\nabla f = 0\}.\]
\end{definition}

\subsection{A family of online \textit{k}-means algorithms}
In this work, we analyze the family of \ref{alg:generalized-online-kmeans} algorithms. In \Cref{sec:motivation}, we motivated the \ref{alg:distributional-kmeans} algorithm, which updates a single center per iteration. The generalized family is a superset of algorithms in which multiple centers can be updated each step. Here, the update to the $i$th center is computed using data drawn from the $i$th Voronoi cell,
$X_i \sim p|_{V_i(W)}$.
\vspace{0.2cm}

\begin{figure*}[h!]
\begin{center}
\begin{minipage}{0.9\textwidth}
\hrulefill
\begin{algorithm}{generalized online \mbox{$k$-means}}{
    \label{alg:generalized-online-kmeans}
    \qinit{$k$ arbitrary distinct centers $W^{(0)} \in \mathbb{R}^{k\times d}$ from the support of $p\vphantom{\big|_p}$}}
    \qfor $n = 0,1,2,\ldots\vphantom{\big|_{p_i}}$\\ 
    \qdo
        sample data points $X^{(n+1)}_i\sim p\big|_{V_{i}(W^{(n)})}$\quad  for $i = 1,\ldots, k$ \\
    update all centers\quad $\displaystyle W^{(n+1)}_i \leftarrow W^{(n)}_i - H_i^{(n+1)}\cdot \big(W_i^{(n)} - X_i^{(n+1)}\big) \vphantom{\big|_{V_{i}(W^{(n)})}}$ 
    \qrof
\end{algorithm}
\hrulefill
\end{minipage}
\vspace{5pt}

\begin{minipage}{0.85\textwidth}
\small{\textsc{algorithm}. The class of algorithms analyzed in this work. It generalizes \ref{alg:distributional-kmeans} by allowing multiple centers to be updated each step. It can be further generalized to an online mini-batch setting, see \Cref{rmk:further-generalization}. To recover the single-center update, see \Cref{rmk:recover-single-update}.} 
\end{minipage}
\end{center}
\end{figure*}

\begin{remark}[A further generalization] \label{rmk:further-generalization}
As \textit{ad hoc} notation for this remark, let $p_i(W) := p|_{V_i(W)}$ so that the update $X_i$ is drawn from $p_i(W)$. The only properties we use about $p_i(W)$ are that:
\begin{enumerate}
    \item[(i)] $p_i(W)$ has mean $M_i(W)$, and
    \item[(ii)] $p_i(W)$ is supported only in the interior of $V_i(W)$.
\end{enumerate}
So, the update distributions $p_i(W)$ may be generalized to any satisfying (i) and (ii). For example, the result of this paper holds for an online mini-batch setting, where the update $X_i$ is computed by averaging multiple draws from $p|_{V_i(W)}$. But let us refrain from adding even more notation and simply assume that the update distributions are $p|_{V_i(W)}$.
\end{remark}

\begin{remark}[Recovering the single-center update] \label{rmk:recover-single-update}
Notice that the earlier \ref{alg:distributional-kmeans} is a specific case of \ref{alg:generalized-online-kmeans}, where the learning rate $(H_1,\ldots, H_k)$ is supported only on the $i$th coordinate---only $H_i$ is nonzero---with probability $P_i(W)$. 
\end{remark}
\section{Main Results}
We prove two main results in this paper: (i) we prove  that \ref{alg:generalized-online-kmeans} asymptotically converges to stationary points of the $k$-means cost under fairly general conditions, and (ii) we extend Lloyd's method into the online setting and show that it satisfies these fairly general conditions; thus, it asymptotically converges. The remainder of this section: (i) \Cref{sec:conv-online} sketches the proof ideas leading to the general convergence result, culminating in \Cref{thm:iterates-conv-simple}; (ii) \Cref{sec:online-lloyd-conv} does the same for the particular  convergence result for our online extension of Lloyd's algorithm in \Cref{thm:lloydconvergence}.

\subsection{Convergence of online \textit{k}-means algorithms}
\label{sec:conv-online}

To prove convergence for \ref{alg:generalized-online-kmeans}, we show that it performs stochastic gradient descent: if the current $k$ centers is $W$, then the $i$th center has expected negative update direction:
\[\E_{X_i \sim V_i(W)}\left[W_i - X_i\,\big|\, W\right] = W_i - M_i(W).\]
In \Cref{sec:grad-k-means-res}, we give a sketch showing that the gradient of the $k$-means cost $f$ at $W$ also points in the same direction, implying that the update is just a noisy gradient descent step:
\[\nabla_{w_i} f(W) \propto W_i - M_i(W).\]
This connection between \ref{alg:generalized-online-kmeans} and SGD allows us to use standard ideas from optimization to prove that the cost converges. That is, there is an $\mathbb{R}$-random variable $f^*$ such that:
\[f(W^{(n)}) \to f^* \quad \mathrm{a.s.}\]
We sketch this in \Cref{sec:cost-conv-sketch}. Finally, we give the asymptotic convergence of the iterates. Here, existing results from optimization were insufficient because the SGD that \ref{alg:generalized-online-kmeans} performs allows the learning rates $({H_1}^{(n)},\ldots, {H_k}^{(n)})$ to be stochastic, non-uniform across centers, and arbitrarily dependent on the past. We sketch the main proof ideas in \Cref{sec:iter-conv-sketch}.

\subsubsection{Gradient of \textit{k}-means objective}
\label{sec:grad-k-means-res}
In order to analyze the $k$-means algorithm through the lens of gradient descent, we need to be able to prove smoothness properties and calculate the gradient of the $k$-means objective,
\begin{equation} 
f(w) := \frac{1}{2}\sum_{i \in [k]}\ \int_{V_i(w)} \| w_i - x\|^2\, p(x) \,dx. \tag{{\ref{eqn:kmeans-cost}}}
\end{equation}
Computing the derivative of the $k$-means cost (\ref{eqn:kmeans-cost}) with respect to $w$ is relatively involved because the both the domain of integration and the integrand depend on $w$. However, the result is simple:

\begin{restatable}[Gradient of $k$-means objective]{lemma}{kmeansgrad} \label{lem:k-means-grad}
Let $p$ be a density on $\mathbb{R}^d$ with $\E_{X \sim p}\left[\|X\|^2\right] < \infty$. Let $f$ be the $k$-means objective (\ref{eqn:kmeans-cost}). Then $f$ is continuously differentiable on $\mathcal{D}$, where:
\begin{align} \label{eqn:grad-f}
    \nabla_{w_i}\, f(w) &= P_i(w) \cdot \big(w_i - M_i(w)\big).
\end{align}
\end{restatable}

\begin{proof}[Proof sketch]
A change in $f$ due to a small perturbation at $w$ to $w + \epsilon$ can be broken down into two parts. First, for points $x \in V_i(w) \cap V_i(w + \epsilon)$ that remain within the $i$th Voronoi region, the accumulated change in cost is due to shifting the $i$th center,
\[\frac{1}{2}\int_{V_i(w) \cap V_i(w+ \epsilon)} \big(\|w_i + \epsilon_i - x\|^2 - \|w_i - x\|^2 \big)\, p(x)\, dx.\]
Note that in the limit as $\epsilon$ goes to 0, the domain of integration is the points in the interior of $V_i(w)$. 

Second, for points $x  \in V_i(w) \cap V_j(w + \epsilon)$ that switch from the $i$th to the $j$th Voronoi region, the change in cost is due to switching regions. Note that in the limit as $\epsilon$ approaches 0, these points are on the boundary $V_i(w) \cap V_j(w)$. But as points on the boundary $V_i(w) \cap V_j(w)$ are equally distant from either the $i$th and $j$th centers, this second term overall contributes nothing to the first-order change in $f$. It turns out that the derivative of $f$ can be computed by treating the domains of integration as fixed. By dominated convergence, we can move the derivative past the integral:
\begin{align*}
    \nabla_{w_i} f(w) &= \frac{1}{2} \int_{V_i(w)} \nabla_{w_i} \|w_i - x\|^2\, p(x)\, dx 
    \\&= \int_{V_i(w)} (w_i - x) \,p(x)\,dx.
\end{align*}
Substituting the definition of $P_i$ and $M_i$ completes the proof. 

\Cref{sec:prop} makes this argument rigorous.
\end{proof}

\subsubsection{Convergence of cost} 
\label{sec:cost-conv-sketch}
For the $k$-means cost to be finite, the distribution $p$ must have bounded second moment. But for our convergence analysis, we make a stronger, but fairly common, bounded support assumption (e.g.\@ \cite{bartlett1998minimax,ben2007framework,paul2021uniform}): 

\begin{assumption} \label{assume:finite}
Assume $p$ has bounded support, i.e. $\Pr_{X \sim p}(\|X\| > R ) = 0$ for some $R > 0$.
\end{assumption}
We also make assumptions on the learning rates used in \ref{alg:generalized-online-kmeans}. Let $X^{(n)}$ and $H^{(n)}$ be the tuples $\big(X_1^{(n)},\ldots, X_k^{(n)}\big)$ and $\big(H_1^{(n)},\ldots, H_k^{(n)}\big)$.

\begin{assumption} \label{assume:lr}
Let $\big(\mathcal{F}_n\big)_{n=0}^\infty$ be the natural filtration associated to the \ref{alg:generalized-online-kmeans} algorithm. That is, let $\mathcal{F}_0 = \sigma\big(W^{(0)}\big)$ be the \mbox{$\sigma$-algebra} generated by $W^{(0)}$, and let the \mbox{$\sigma$-algebra} $\mathcal{F}_n = \sigma\big(X^{(n)}, H^{(n)}, \mathcal{F}_{n-1}\big)$ contain all information up to iteration $n$. Assume:
\begin{enumerate}
    \item[(1)] If $P_i(W^{(n)}) = 0$, then $H_i^{(n+1)} = 0$ almost surely.
    \item[(2)] $H^{(n+1)}$ and $X^{(n+1)}$ are conditionally independent given $\mathcal{F}_n$.
    \item[(3)] $0 \leq H_i^{(n+1)}\leq 1$.
\end{enumerate}
\end{assumption}

The first assumption helps us avoid the ill-defined situation of drawing from $V_i(W^{(n)})$ when $P_i(W^{(n)}) = 0$. We require ${H_i}^{(n+1)} = 0$ a.s.\@ so that ${X_i}^{(n+1)}$ may be arbitrary, since it goes unused in the update. The second assumption ensures that $H^{(n+1)}$ does not depend on $X^{(n+1)}$, so that the update direction is an unbiased estimate of gradient descent. The final assumption is simplifying and natural: the update is a convex combination of the previous center and new data point. It also follows that $W^{(n)}$ remains in $\mathcal{D}$, the set of non-degenerate tuples almost surely:
\begin{restatable}{lemma}{lemnondegenerate} \label{lem:non-degenerate}
Let $0 \leq H_i^{(n+1)} \leq 1$. Then $W^{(n)} \in \mathcal{D}$ is non-degenerate for all $n \in \mathbb{N}$ almost surely.
\end{restatable}

 With \Cref{assume:finite}, we also have that all centers and updates takes place in the closed ball $B(0,R) \subset \mathbb{R}^d$ of radius $R$ almost surely. As this is a region with diameter $2R$, the amount that each center moves can be controlled by bounding the learning rates:

\begin{restatable}{lemma}{lemlrbounditers} \label{lem:lr-iterates-bound}
Let $\Pr_{X \sim p}(\|X\| > R ) = 0$. Let $H_i^{(n)}\leq 1$ for all $n \in \mathbb{N}$ and $i \in [k]$. If $n > m$, then:
\[\|W^{(m)} - W^{(n)}\| \leq 2R \cdot \sum_{i \in [k]} \sum_{m \leq n' < n} H_i^{(n'+1)}\quad \mathrm{a.s.}\]
\end{restatable}

From now on, we implicitly make Assumptions \ref{assume:finite} and \ref{assume:lr}. With them, the convergence of the $k$-means cost $f(W^{(n)})$ follows from a supermartingale argument commonly used in proving the convergence of stochastic gradient descent. Recall that a bounded and monotonically decreasing real-valued sequence converges to a real value. This remains true in the random setting. A supermartingale is a noisy sequence that in expectation decreases monotonically. Provided that the noise can be controlled, then the \emph{martingale convergence theorem} shows that such a stochastic sequence will converge to some real-valued random variable, e.g.\@ see \cite{durrett2019probability}.

\begin{restatable}[Convergence of cost]{proposition}{propcostconv}
\label{prop:cost-conv}
Let $\big(W^{(n)}\big)_{n=0}^\infty$ be a sequence generated by \ref{alg:distributional-kmeans}. If the following series converges:
\[\quad\sum_{n=1}^\infty \sum_{i \in [k]} \big(H_i^{(n)}\big)^2 < \infty \quad \mathrm{a.s.},\]
then there is an $\mathbb{R}$-valued random variable $f^*$ such that $f(W^{(n)})$ converges to $f^*$ almost surely.
\end{restatable}
\begin{proof}[Proof sketch]
\Cref{lemma:quad-upper-bound} shows that $f$ has a quadratic upper bound tangent at any $w \in \mathcal{D}$,
\[f(w') \leq f(w) + \nabla f(w)^\top (w' - w) + \frac{1}{2} \|w' - w\|^2.\]
To lower bounds the amount $f$ decreases at each iteration, let $w' = W^{(n+1)}$ and $w = W^{(n)}$. Recall:
\[W_i^{(n+1)} - W_i^{(n)} = - H_i^{(n+1)} \cdot \underbrace{\big(W_i^{(n)} - X_i^{(n+1)}\big)}_{\textrm{noisy direction}}.\]
This implies that in expectation, the algorithm takes a step in the direction of the negative gradient: 
\[\nabla_{w_i} f(W^{(n)}) = P_i(W^{(n)}) \cdot \underbrace{\big(W_i^{(n)} - M_i(W^{(n)})\big)}_{\textrm{expected direction}}.\]

Suppose we are able to neglect the quadratic term $\frac{1}{2} \|w
' - w\|^2$ of the upper bound. Then this shows that $f(W^{(n)})$ is a supermartingale; $f$ decreases in expectation each iteration since it takes a step in the direction of the negative gradient. In order to apply the martingale convergence theorem, we need to know not only that the expected value decreases, but that the total amount of noise is bounded. Indeed, the noise at each step can be bounded by \Cref{lem:lr-iterates-bound}. In particular,
\begin{align*} 
\mathrm{Var}\big(\|W_i^{(n+1)} - W_i^{(n)}\|\big) &\leq \sum_{i \in [k]} \E\left[\big\|W_i^{(n+1)} - W_i^{(n)}\big\|^2\right] 
\\&\leq 4R^2 \sum_{i \in [k]} \big(H_i^{(n+1)}\big)^2.
\end{align*}
Thus the total noise of the process is finite:
\[\sum_{n \in \mathbb{N}} \mathrm{Var}\big(\|W^{(n+1)} - W^{(n)}\|\big) < \infty\qquad \mathrm{a.s.}\] 
Martingale convergence shows that $f(W^{(n)})$ converges. 

Of course, we cannot just drop the quadratic term in the upper bound. But notice that the quadratic terms form convergent series; each term is dominated by terms of a convergent series:
\[\frac{1}{2} \big\|W^{(n+1)} - W^{(n)}\big\|^2 \leq 2R^2 \sum_{i \in [k]} \big(H_i^{(n+1)}\big)^2.\]
Therefore, except for a convergent series, $f(W^{(n)})$ is a supermartingale. Martingale convergence applies here too. \Cref{sec:cost-conv} fills in the technical details.
\end{proof}

\subsubsection{Convergence of iterates}
\label{sec:iter-conv-sketch}
Our first main result is the convergence of the $k$-means iterates $W^{(n)}$ to the set of stationary points $\{\nabla f = 0\}$. For this, we need an additional assumption. Recall that \Cref{lem:non-degenerate,lem:lr-iterates-bound} show that the iterates $W^{(n)}$ are non-degenerate and bounded---they are contained in $\mathcal{D}_R$ (defined in \Cref{sec:prelim}). From now on, we restrict all sets to the topological subspace $\mathcal{D}_R \subset \mathbb{R}^{k \times d}$. For example, when we write $\{\nabla f = 0\}$, we implicitly mean $\{\nabla f = 0\} \cap \mathcal{D}_R$. We assume:

\begin{assumption}\label{assume:compact-stationary}
The set $\{\nabla f = 0\}$ of stationary points is compact in $\mathcal{D}_R$.
\end{assumption}

Geometrically, this means that $\{\nabla f = 0\}$ has no degenerate limit point (see \Cref{lem:compact-eps}'s proof). This lets us prove that $W^{(n)}$ converges to $\{\nabla f = 0\}$ by showing that $\nabla f(W^{(n)})$ converges to zero. Formally, we apply \Cref{lem:top-conv}, which relies on (i) the continuity of $\nabla f$, and (ii) the existence of a compact subset of $\mathcal{D}_R$ containing $\{\nabla f = 0\}$. Indeed, a consequence is that level sets $\{\|\nabla f \| \leq \epsilon\}$ are compact for small enough $\epsilon$.

\begin{restatable}{lemma}{lemcompacteps} \label{lem:compact-eps}
Let $\{\nabla f = 0\}$ be compact in $\mathcal{D}_R$. There exists $\epsilon_0 > 0$ so that if $\epsilon \in [0,\epsilon_0]$, the sets $\{\|\nabla f \| \leq \epsilon\}$ and $\{\|\nabla_{w_i} f\| \leq \epsilon\}$ are compact in $\mathcal{D}_R$ for $i \in [k]$. 
\end{restatable}

From here on, we additionally make \Cref{assume:compact-stationary}.

\paragraph{Proof ideas} The key idea is to link the convergence of $\nabla f(W^{(n)}) \to 0$  to that of \mbox{$f(W^{(n)}) \to  f^*$}. Broadly speaking, we need to impose conditions on the learning rate so that whenever $\|\nabla f(W^{(n)})\|$ is larger than some $\epsilon > 0$, then the cost likely decreases by some $\delta > 0$. The Borel-Cantelli lemma (\Cref{lem:borel-cantelli}) would then show that if the gradient is large infinitely often, then the cost must also decrease by a large amount infinitely often. But as the cost converges, this cannot happen.

Were we in the noiseless gradient descent setting, we could---in a single step---turn a large gradient into a large decrease in cost. We simply ensure that the iterates move in a direction that significantly decreases $f$ by lower bounding the learning rates for centers $i \in [k]$ with large gradients $\nabla_{w_i} f$.  However, in our setting, the true gradient direction is obscured by the presence of noise, and so lower bounding the learning rates by a constant no longer guarantees a decrease in cost.

Instead, in the stochastic gradient descent setting, we can---over many steps---turn a \emph{region} with large gradients into a large decrease in cost, with high probability. Since our learning rate decays to zero, eventually, whenever we enter a region with large gradients, we remain in that region for sufficiently many iterations so as to average out the noise and recover the underlying signal to move along the negative gradient direction. As in the noiseless setting, we want the choice of learning rate to not disproportionately dampen learning for centers $i \in [k]$ with large gradients $\nabla_{w_i} f$. In other words, the iterates should not leave this region of large gradients before having accumulated enough learning on those centers with large gradients. 

As a motivating adversarial example, suppose we decrease the learning rates for centers with large gradients while increasing the rates for centers with small gradients. Then, the iterates may escape the region with large gradients by moving orthogonally to the gradient descent direction, not significantly decreasing the function value. And as an extreme example, if we never update the $i$th center, so $H_i^{(n+1)} \equiv 0$ for all $n$, then the cost $f(W^{(n)})$ may converge while the gradients $\|\nabla f(W^{(n)})\|$ may never converge to zero. 

To preclude these examples, we need to be able to control the learning rates whenever the gradient $\|\nabla_{w_i} f(W^{(m)})\|$ becomes large for center $i \in [k]$ at iteration $m$. But because the convergence of cost analysis required the learning rates to decay to zero (\Cref{prop:cost-conv}), to accumulate the same amount of learning, we will need to control the learning rates over increasingly many iterations as $m \to \infty$. In particular, if the gradient becomes large at iteration $m$, we shall aim to control the learning rates during the interval between $m$ and a future horizon $T(m)$, for some appropriately chosen function $T : \mathbb{N} \to \mathbb{N}$.

We impose two types of conditions on the learning rate. The first condition allows us to ensure that the iterates remain within the region of large gradients during the interval from $m$ to $T(m)$. Recall \Cref{lem:lr-iterates-bound} showed that we can bound the total displacement between $W^{(m)}$ and $W^{(T(m))}$ by bounding the accumulated learning rates; the first condition is of the form:
\[\sum_{j \in [k]} \sum_{m \leq n < T(m)} H_j^{(n+1)} < r.\]
 The second condition ensures that we accumulate enough learning for the center $i \in [k]$ with large gradients, so that the cost decreases by a constant between iterations $m$ and $T(m)$,
\[\sum_{m \leq n < T(m)} H_i^{(n+1)} > s.\]
In fact, it is enough that these conditions hold with constant probability.

\begin{restatable}[Convergence of iterates]{theorem}{thmiteratesconv}  \label{thm:iterates-conv-simple}
Let $W^{(n)}$ and $H^{(n+1)}$ be as in \Cref{prop:cost-conv}. Suppose that for all $i \in [k]$, $\epsilon > 0$, and sufficiently small $r > 0$, there exists $T : \mathbb{N} \to \mathbb{N}$, $m_0 \in \mathbb{N}$, and some $s, c > 0$ so that for any $m > m_0$,
\begin{align*}
\Pr\Bigg(\sum_{j \in [k]}\sum_{m \leq n < T(m)}  H_j^{(n+1)}  < r \quad \mathrm{and} \sum_{m \leq n < T(m)} H_i^{(n+1)} > s 
\, \Bigg|\, \mathcal{F}_m, \|\nabla_{w_i} f(W^{(m)})\| > \epsilon \Bigg) > c,
\end{align*}
Then, $W^{(n)}$ asymptotically converges to stationary points of the $k$-means cost $f$ almost surely.
\end{restatable}

\begin{proof}[Proof sketch]
We show for any $\epsilon > 0$ and $i \in [k]$, the gradient $\|\nabla_{w_i} f(W^{(n)})\|$ is greater than $\epsilon$ only finitely often. Suppose that for all $w \in \mathcal{D}_R$ with large gradient $\|\nabla_{w_i} f(w)\| > \epsilon$, there is a region around $w$ also with large gradients. In particular, assume for some $R_0$ depending on $\epsilon$ that: 
\[\|\nabla_{w_i} f(w')\| >\frac{\epsilon}{2} \qquad \textrm{if} \quad \|w - w'\| < R_0.\]

Let $r_0 = R_0/2R$ and choose $r < r_0$. Then, we obtain:
\[\|W^{(m)} - W^{(n)}\| < R_0,\]
for all $m \leq n < T(m)$ whenever the first event in the probability in the theorem statement holds. This follows from \Cref{lem:lr-iterates-bound}, which bounds this displacement through the learning rate bound $r_0$. In short, with constant probability during this interval, the $i$th gradient remains large, \[\|\nabla_{w_i} f(W^{(n)})\| > \frac{\epsilon}{2}.\] 

Since the $i$th gradient remains large, the cost is likely to decrease significantly as long as the $i$th center is updated enough times. \Cref{lem:supermartingale-approx} relates the learning rate to the decrease in cost, so that if the second event in the probability statement holds, then the cost decreases by at least: 
\[\sum_{m \leq n < T(m)} H_i^{(n+1)} \|\nabla_{w_i} f(W^{(n)})\|^2 > \frac{s \epsilon^2}{4},\]
if we may neglect the noise terms of \Cref{lem:supermartingale-approx}. And so if $\|\nabla_{w_i} f(W^{(n)})\| > \epsilon$ infinitely often, then the cost must decrease by a constant infinitely often by Borel-Cantelli, contradicting the convergence of the cost. 

The proof is only slightly more complicated because of noise, but not significantly so as the noise forms a convergent series (\Cref{lemma:convergence-n-b}). As a result, there is a random time $M$ after which a deterministic cost decrease of $s\epsilon^2/4$ will completely dominate any increase of cost due to noise.

The remaining technical complication is our implicit assumption earlier that $\|\nabla_{w_i} f\|$ is uniformly continuous on $\mathcal{D}_R$, to allow us to find a constant $R_0$ given $\epsilon$ that holds for all $\|\nabla_{w_i}f(w)\| > \epsilon$. Though $w \mapsto \|\nabla_{w_i} f(w)\|$ is continuous on $\mathcal{D}_R$, it may not be uniformly so since $\mathcal{D}_R$ is not compact. Our solution is to apply \Cref{lem:compact-eps} to find an $\epsilon_0 > 0$ for which $\{\|\nabla_{w_i} f\| \leq \epsilon_0\}$ is compact, so:
\begin{equation} \label{eqn:dr-break}
\mathcal{D}_R = \{\|\nabla_{w_i} f\| \leq \epsilon_0\} \cup \{\|\nabla_{w_i} f\| > \epsilon_0\}.
\end{equation}
The first subset is compact, so uniform continuity holds. On the second subset, we rule out the possibilities that (i) the iterates eventually remain in this set, and (ii) the iterates exit and re-enter this set infinitely often. Case (i) is impossible by an argument akin to our earlier one on $[\epsilon,\epsilon_0]$. In fact, here we can bypass finding $R_0$ altogether since the iterates are guaranteed to remain in a region with large gradients forever. Case (ii) is impossible because this forces the iterates to enter $[\epsilon,\epsilon_0]$ infinitely often once the learning rate has become sufficiently small.  Thus for all $\epsilon > 0$, the iterates eventually never return to $\{\|\nabla_{w_i} f\| > \epsilon\}$.
\end{proof}

We also prove a slightly more general \Cref{thm:iterates-conv}, which makes the decomposition (\ref{eqn:dr-break}) explicit. We introduce separate conditions on the learning rate on $\{\|\nabla_{w_i} f\| \in [\epsilon,\epsilon_0]\}$ and $\{\|\nabla_{w_i} f\| > \epsilon_0\}$; that way, we can remove the learning rate upper bound condition where we do not need it.

\subsection{Convergence of an online Lloyd's algorithm}
\label{sec:online-lloyd-conv}

Let us return now to extending Lloyd's algorithm into the online setting.
Lloyd's algorithm is defined to iteratively set each center $w_i$ to $M_i(w)$,
\[w_i \leftarrow  M_i(w).\]
Since $\nabla_{w_i} f(w) = P_i(w) \cdot \big(w_i - M_i(w)\big)$, the Lloyd update is precisely \textit{preconditioned gradient descent}, where the preconditioner at $w$ for the $i$th center is $P_i(w)^{-1}$. 
\begin{equation*}  \label{eqn:batch-update}
w_i \leftarrow w_i - \frac{1}{P_i(w)} \nabla_{w_i} f(w).
\end{equation*}
A noiseless gradient descent algorithm is often converted into a stochastic one by introducing decaying learning rates \citep{bottou1998online}, which allows the noise to average out over time. Applying the common decay rate of $n^{-1}$, we might aim for an update that, in expectation, looks like:
\[w_i \leftarrow w_i - \frac{1}{n P_i(w)} \nabla_{w_i} f(w).\]

\paragraph{A first pass at online Lloyd's} To design the stochastic version of Lloyd's algorithm, suppose that we had access directly to $p$. Then, we might consider the following idealized learning rate achieving the above expected update:
\[H_{\mathrm{ideal},i}^{(n+1)} \leftarrow  \frac{\ind\{I^{(n+1)} = i\}}{nP_i(W^{(n)})},\]
where $I^{(n+1)}$ is the index of the updated center; it is drawn from the distribution $P(W^{(n)})$ over $[k]$.

The \ref{alg:naive} algorithm described in \Cref{sec:motivation} could be considered as a rough approximation to this idealized learning rate. Recall its update:
\begin{equation*}
H_{\mathrm{OL},i}^{(n+1)} \leftarrow \frac{\ind\{I^{(n+1)} = i\}}{N_i^{(n)} + 1},
\end{equation*}
where $N_i^{(n)}$ is the update count for the $i$th center. If $P(W^{(n)})$ does not vary much over time, then:
\[\frac{1}{N_i^{(n)} + 1} \approx \frac{1}{nP_i(W^{(n)})}.\] 
In other words, the online Lloyd's update appears to approximate the idealized learning rate by applying a stochastic preconditioner computed on:
\[\widehat{P}_i^{(n)} = \frac{N_i^{(n)}}{n}.\]
This algorithm naively assumes that $\widehat{P}_i^{(n)}$ is a reasonable estimator of ${P_i}^{(n)} := P_i(W^{(n)})$. Because the Voronoi partitions is drifting throughout the whole algorithm, this assumption is generally false.

However, the issue with \ref{alg:naive} is not its naivet\'e. Rather, the issue lies with the idealized preconditioner $P_i(w)^{-1}$, which is poorly conditioned---it can become arbitrarily large. While not a problem in the noiseless setting, the learning rate cannot become unbounded in the stochastic case.

\paragraph{A second pass at online Lloyd's} Let us consider a second idealized online Lloyd's where we introduce an upper bounding rate of $t_n^{-1}$ for some sequence $t_n$. Set the learning rate:
\[H_{\mathrm{ideal}',i}^{(n+1)} \leftarrow  \frac{\ind\{I^{(n+1)} = i\}}{\max\big\{nP_i^{(n)},\, t_n\big\}}.\]
We refrain from preconditioning when $P_i(w) < t_n / n$. Note that if $t_n = o(n)$, then we can expect the set of points on which the algorithm will precondition to grow. Of course, this is idealized since we generally do not have direct access to $p$ to compute $P_i(w)$ with.

This motivates what we call the \ref{alg:onlinelloyd} algorithm, which simultaneously constructs an estimator ${\widehat{P}_i}^{(n)}$ of ${P_i}^{(n)}$ based on the empirical rate at which the $i$th center has recently been updated in the past $s_n$ steps, for some sequence $s_n$. Then, we set:
\begin{equation} \label{eqn:learning-rate-empirical}
H_i^{(n+1)} \leftarrow  \frac{\ind\{I^{(n+1)} = i\}}{\max\big\{n{\widehat{P}_i}^{(n)},\, t_n\big\}}.
\end{equation}On the one hand, in order to obtain a low-bias estimator, we need $s_n = o(n)$ so that updates in the distant past are forgotten. But on the other, we would like the estimator to concentrate as $n$ goes to infinity, so we also require $s_n \uparrow \infty$. To specify ${\widehat{P}}_i^{(n)}$, we define:
\begin{align}\label{eqn:learning-rate-empirical-p}
    \widehat{P}_i^{(n)} &= \frac{1}{s_n} \sum_{n_\circ\leq n' < n} \ind\{I^{(n' + 1)} = i\}
\end{align}
where we denote $n_\circ = n - s_n$, and where $s_n$ and $t_n$ are non-decreasing sequences. 

\begin{figure*}[h!]
\begin{center}
\begin{minipage}{0.9\textwidth}
\hrulefill
\begin{algorithm}{generalized online Lloyd's}{
    \label{alg:onlinelloyd}
    \qinit{$k$ arbitrary distinct centers $W \in \mathbb{R}^{k\times d}$ from the support of $p\vphantom{\big|_p}$}}
    \qfor $n = 0,1,2,\ldots\vphantom{\big|_{p_i}}$ \\
        sample data point $X\sim p\vphantom{\displaystyle\frac{1}{s}}$\\
        let $\displaystyle W_{i}$ be the closest center to $X$\\ 
         update $\displaystyle \widehat{P}_{i} \leftarrow \frac{1}{s_n} \cdot \text{\# of times ith center was updated in last}~s_n~\text{timesteps}$\\   
    update center $\displaystyle W_{i} \leftarrow W_{i} - \frac{1}{\max\{n \widehat{P}_i, t_n\}}\cdot \big(W_{i} - X\big) $
   
    \qrof
\end{algorithm}
\hrulefill
\end{minipage}
\vspace{5pt}

\begin{minipage}{0.85\textwidth}
\small{\textsc{algorithm}. A generalization of \ref{alg:naive} with asymptotic convergence to stationary points guarantees, for example, when $s_n = n^{2/3 + \epsilon}$ and $t_n = n^{2/3 + 2 \epsilon}$ for $\epsilon > 0$ (\Cref{thm:lloydconvergence}). Note that \ref{alg:naive} is recovered when $s_n = n$ and $t_n = 0$, but has no guarantees.} 
\end{minipage}
\end{center}
\end{figure*}

By introducing conditions on $s_n$ and $t_n$,  we show that ${\widehat{P}_i}^{(n)}$ is a consistent estimator of ${P_i}^{(n)}$ in \Cref{sec:lloyd}. Furthermore, we obtain the following convergence theorem, as a corollary of \Cref{thm:iterates-conv}.

\begin{restatable}[Convergence of iterates, \ref{alg:onlinelloyd}]{theorem}{lloydconvergence} \label{thm:lloydconvergence}
Let $H_j^{(n)}$ and $\widehat{P}_j^{(n)}$ be defined as in (\ref{eqn:learning-rate-empirical} and \ref{eqn:learning-rate-empirical-p}). Let $s_n$ and $t_n$ be non-decreasing sequences satisfying:
\begin{align*}\lim_{n\to\infty}\, \frac{n^{2/3} \log n}{s_n} = \lim_{n\to\infty} \, \frac{s_n \log s_n}{t_n} =  \lim_{n\to\infty}\,\frac{t_n}{n} = 0.\end{align*}
If $p$ is continuous, then the iterates $W^{(n)}$ of \ref{alg:onlinelloyd} asymptotically converges to stationary points of its $k$-means cost almost surely.
\end{restatable}

\begin{proof}[Proof sketch]
To apply \Cref{thm:iterates-conv-simple}, we need to show that with constant probability:
\begin{align*} 
\sum_{j \in [k]} \sum_{m \leq n < T(m)} H_j^{(n+1)} < r
\quad \textrm{and}\qquad \sum_{m \leq n < T(m)} H_i^{(n+1)} > s,
\end{align*}
whenever $\|\nabla_{w_i} f(W^{(m)})\| \geq \epsilon$. We achieve this by proving tighter bounds \emph{in expectation}:
\begin{align*}
&\sum_{j \in [k]} \sum_{m\leq n < T(m)} \E\left[H_j^{(n+1)}\,\middle|\, \textcolor{gray}{\mathcal{F}_n}\right]  < r/100
\qquad \textrm{and}\qquad \sum_{m \leq n < T(m)} \E\left[H_i^{(n+1)}\,\middle|\, \textcolor{gray}{\mathcal{F}_n}\right]  > 100s,
\end{align*}
which can be converted into bounds \emph{in probability} by Markov's and Azuma-Hoeffding's inequalities.

Notice that the conditional expectation is:
\[\E\left[H_j^{(n+1)}\,\middle|\, \textcolor{gray}{\mathcal{F}_n}\right] = \frac{P_j^{(n)}}{\max\{n \widehat{P}_j^{(n)}, t_n\}}.\]
Assume $P_j^{(n)}$ is not too small and that $\widehat{P}_j^{(n)} = P_j^{(n)}$, so the estimator is perfect. Then in fact:
\[\E\left[H_j^{(n+1)}\,\middle|\, \textcolor{gray}{\mathcal{F}_n}\right] = \frac{1}{n}.\]
Define the map $T_r : \mathbb{N} \to \mathbb{N}$ so that $T_r(m)$ is the unique natural number so that:
\begin{equation*}
    \sum_{m \leq n \textcolor{blue}{<} T_r(m)} \frac{1}{n} \leq r < \sum_{m \leq n \textcolor{orange}{\leq} T_r(m)} \frac{1}{n}.
\end{equation*}
We obtain the conditional expectation bounds if we set $T\equiv T_{Cr}$ with $C = 1/100k$ and $s = Cr/100$. 

We assumed (i) $P_j^{(n)}$ is not too small and (i) $\widehat{P}_j^{(n)}$ is perfect; neither is true in general. The first caveat is easier to deal with; the fear is that the centers with low update probabilities may sporadically contribute a large amount $1/t_n$, when compared to the $1/n$ of the other centers. However, our conditions on $s_n$ and $t_n$ ensure that centers with small Voronoi masses also have little effect on the overall behavior. For the second, we show that the estimator is highly concentrated around its true mean (\Cref{lem:est-concentration}). 

To sketch why $\widehat{P}_j^{(n)}$ concentrates about $P_j^{(n)}$, we have:
\begin{align*}
\widehat{P}_j^{(n)} = \frac{1}{s_n} \sum_{n_\circ\leq n' < n} \ind\{I^{(n' + 1)} = j\}
\qquad \textrm{and}\qquad P_j^{(n')} = \E\left[\ind\{I^{(n'+1)} = j\}\,\middle|\, \mathcal{F}_{n'}\right].
\end{align*}
Azuma-Hoeffding's shows that $\widehat{P}_j^{(n)}$ concentrates:
\begin{align*}
\Pr\left(\bigg|\widehat{P}_j^{(n)} -  \frac{1}{s_n} \sum_{n_\circ \leq n' < n} P_j^{(n')}\bigg| \geq a \,\middle|\, \textcolor{gray}{\mathcal{F}_{n_\circ}}\right) \leq 2 \exp\left(- \frac{1}{2} s_n a^2\right).
\end{align*}
We just need to show that $P_j^{(n')}$ remains close to $P_j^{(n)}$ over this interval. Suppose that $P_j$ were $L$-Lipschitz. Then for all $n_\circ \leq n' < n$,
\[\left|P_j^{(n')} - P_j^{(n)}\right| \leq 2RL \cdot \sum_{j^{\prime \prime} \in [k]}\sum_{n_\circ \leq n^{\prime\prime}< n} H_{j^{\prime \prime}}^{(n^{\prime \prime}+1)}.\]
The right-hand side goes to zero almost surely under our conditions on $s_n$ and $t_n$ (\Cref{lem:lr-ub-sn}).

This would be the proof if $P_j$ were globally Lipschitz on $\mathcal{D}_R$. But this is not generally the case: consider the 2-means problem on $\mathbb{R}^2$. When the two centers are very close together, the Voronoi cells they produce are much more sensitive to small perturbations than when they are far apart.

However, it is the case that when $p$ is continuous, then $P_j : \mathcal{D}_R \to [0,1]$ is locally Lipschitz (\Cref{lem:local-lipschitz}), so that $P_j$ is Lipschitz on compact subsets $K$ of $\mathcal{D}_R$. 

A slight complication arises in each step of this proof because cannot directly condition on the iterates remaining in $K$; conditioning on a future event destroys the martingale property required by Azuma-Hoeffding's. To overcome this issue, we construct a core-set $K_\circ \subset K$ so that if $W^{(n_\circ)}$ is initially in $K_\circ$ and the accumulated learning rate is bounded:
\[W^{(n_\circ)} \in K_\circ \quad\textrm{and}\quad \sum_{j \in[k]} \sum_{n_\circ \leq n' < n} H_j^{(n'+1)} < r_\circ,\]
then $W^{(n')}$ remains in $K$ almost surely throughout the interval $n_\circ \leq n' \leq n$. As this construction splits the analysis into two cases (iterates in and outside $K$), in the actual proof, this theorem follows from the more general \Cref{thm:iterates-conv}, which breaks the analysis down into these two cases.

\Cref{sec:lloyd} makes this argument rigorous.
\end{proof}
\section{Analysis of the \textit{k}-means cost} \label{sec:prop}

In order to analyze the $k$-means algorithm through the lens of gradient descent, we need to be able to prove smoothness properties and calculate the gradient of the $k$-means objective,
\begin{equation} 
f(w) := \frac{1}{2}\sum_{i \in [k]}\ \int_{V_i(w)} \| w_i - x\|^2\, p(x) \,dx. \tag{{\ref{eqn:kmeans-cost}}}
\end{equation}
But because the domains and integrands depend on $w$ simultaneously, taking the gradient is not so straightforward. To simplify analysis, we can fix the Voronoi partition with respect to some tuple of centers $w' \in \mathbb{R}^{k \times d}$ in order to define a family of surrogate objectives parametrized by $w'$,
\begin{equation} \label{eqn:kmeans-decouple}
    g(w; w') := \frac{1}{2}\sum_{i\in [k]} \ \int_{V_i(w')} \,\|w_i - x\|^2\, p(x)\, dx.
\end{equation}

\subsection{Family of upper bounds of the cost}
It turns out that $\{g(\,\cdot\,;w') : w' \in \mathbb{R}^{k\times d}\}$ forms a family of convex, quadratic upper bounds of $f$. That $g(\,\cdot\,;w')$ is convex and quadratic is easy to see, since it is a sum of convex combinations of convex quadratic functions. The following proposition shows that $g$ dominates $f$.

\begin{proposition} \label{prop:surrogate}
Let $p$ be a density on $\mathbb{R}^d$ with bounded second moment. Let $f$ be the $k$-means objective (\ref{eqn:kmeans-cost}) and $g$ be the $k$-means surrogate (\ref{eqn:kmeans-decouple}). Then for all $w,w' \in \mathbb{R}^{k \times d}$,
\[f(w) \leq g(w;w').\]
\end{proposition}
\begin{proof}
Notice that by definition, $f(w) = g(w;w)$. We claim that:
\begin{equation} \label{eqn:g-min-w}
\textcolor{gray}{f(w) = }\ g(w;w) \leq \, g(w;w').
\end{equation}
To see this, note that $g$ is an integral accumulating $\|w_i - x\|^2$ when $x$ is in the $i$th partition. The integral only decreases if $x$ moves into its Voronoi partition, $j^* = \argmin \|w_j - x\|^2$,
\begin{align*} 
g(w;w') &= \frac{1}{2} \int_{\mathbb{R}^{k\times d}} \, \sum_{i \in [k]} \|w_i - x\|^2 \cdot \ind_{V_i(w')}(x) \, p(x)\,dx
\\&\geq \frac{1}{2} \int_{\mathbb{R}^{k\times d}} \min_{i \in [k]} \,\|w_i - x\|^2\, p(x)\, dx
\\&= g(w;w),
\end{align*}
where $\ind_{V_i(w')}(x)$ is the indicator on the set $V_i(w')$.  The inequality holds because the first integrand dominates the second.
\end{proof}

\subsection{Gradient of the cost}
While taking the derivative of $f$ is nontrivial, taking the derivative of $g$ is much easier; by dominated convergence, the derivative with respect to $w$ is:
\begin{align}
    \nabla_{w_i} g(w;w') &= \int_{V_i(w')} (w_i - x)\, p(x)\, dx \notag\\
    &= P_i(w') \cdot \big(w_i - M_i(w')\big). \label{eqn:grad-g}
\end{align}
We provide two proofs computing the gradient of $f$: (i) an elementary proof based on a local approximation $f(w + \epsilon) = g(w+\epsilon;w) + \mathrm{error}_w(\epsilon)$, and (ii) a short proof using results from differential geometry. Here is our target:

\kmeansgrad*

Given the form of the gradient (\ref{eqn:grad-f}), it is straightforward to show continuity:

\begin{proof}[Proof of continuous gradient]
Assuming (\ref{eqn:grad-f}), the following holds:
\begin{align*}
    \|\nabla_{w_i} f(w + \epsilon) & - \nabla_{w_i}f(w)\| \\&= \big\|\int_{V_i(w + \epsilon)} (w_i + \epsilon_i - x) \, p(x) \, dx - \int_{V_i(w)} (w_i - x)\,p(x)\, dx\big\| 
    \\&\leq \int_{V_i(w + \epsilon) \cap V_i(w)} \|\epsilon_i\|\, p(x) \, dx + 2 \int_{V_i(w + \epsilon) \Delta V_i(w)} \big(\|w_i - x\| + \|\epsilon_i\|\big)\, p(x) \, dx, 
\end{align*}
where $\Delta$ is the symmetric difference operator. Taking as $\|\epsilon\| \to 0$, the limit of both integrals converge to zero by dominated convergence, proving continuity. 
\end{proof}

We now show (\ref{eqn:grad-f}) through two different approaches.

\subsubsection{An elementary proof}
If $\epsilon$ is a small perturbation, we can write $f(w + \epsilon) = g(w + \epsilon;w) + \mathrm{error}_w(\epsilon)$, where the error is accumulated over points near the boundaries of the partitions $V_i(w)$. In the proof, we show that $\mathrm{error}_w(\epsilon) = o\big(\|\epsilon\|\big)$, so only the $g(w + \epsilon;w)$ term contributes to the derivative of $f(w)$. 

\begin{proof}[Proof of \Cref{lem:k-means-grad}]
Fix $w, h \in \mathbb{R}^{k \times d}$ where $\|h\| = 1$. We proceed by computing the directional derivative $D_hf(w)$ along $h$. Notice that:
\begin{align*}
    f(w) = \frac{1}{2} \sum_{i \in [k]}\, \int_{V_i} \|w_i - x\|^2 \, p(x)\,dx \quad\textrm{and}\quad 
    f(w + th) = \frac{1}{2} \sum_{i \in [k]} \, \int_{V_i^t} \| w_i + t h_i - x\|^2\, p(x) \, dx,
\end{align*}
where we let $V_i$ and $V_i^t$ respectively abbreviate $V_i(w)$ and $V_i(w + th)$ for $t > 0$. Notice that:
\[V_i^t \ =\  \big[\textcolor{blue}{V_i} \ \cup \ \textcolor{orange}{\big(V_i^t \setminus V_i\big)}\big]\ \setminus \ \textcolor{purple}{\big(V_i \setminus V_i^t\big)},\]
where (i) $\textcolor{blue}{V_i}$ and $\textcolor{orange}{\big(V_i^t \setminus V_i\big)}$ are disjoint and (ii) $\textcolor{blue}{V_i}$ contains $\textcolor{purple}{\big(V_i \setminus V_i^t\big)}$. It follows that:
\begin{align*}
f(w + th) = \frac{1}{2} &\sum_{i \in [k]}\, \int_{\textcolor{blue}{V_i}} \|w_i + th_i - x\|^2 \,p(x)\, dx \\
&+ \frac{1}{2} \sum_{i \in [k]}\, \left( \int_{\textcolor{orange}{V_i^t \setminus V_i}} \|w_i + th_i - x\|^2 \,p(x)\, dx \ - \ \int_{\textcolor{purple}{V_i \setminus V_i^t}} \|w_i + th_i - x\|^2 \,p(x)\, dx\right).
\end{align*}
We claim that this second line is $o(t)$. Assuming this for now, it follows that the derivative is:
\begin{align*}
    D_hf(w) &= \lim_{t\to 0}\,\frac{1}{t} \big(f(w + th) - f(w)\big)
    \\&= \lim_{t\to 0} \, \frac{1}{t} \left(\frac{1}{2} \sum_{i \in [k]}\, \int_{V_i}\big( \|w_i + th_i - x\|^2 - \|w_i - x\|^2\big) \, p(x)\, dx\right)  + \lim_{t\to 0} \frac{o(t)}{t} = D_hg(w;w),
\end{align*}
where the derivative $D_hg(w;w)$ is taken only over the first argument (i.e. the partition $V_i$ is fixed). But this implies that $\nabla_w\, f(w) = \nabla_w\, g(w;w')$ when $w' = w$. Applying (\ref{eqn:grad-g}), we obtain:
\[\nabla_{w_i}\, f(w) = P_i(w) \cdot \big(w_i - M_i(w)\big).\]

To finish the proof, we need to show that the above second line is $o(t)$. We claim it is equal to:
\begin{equation} \label{eqn:perturbation-term} 
\frac{1}{2} \sum_{i \ne j} \int_{V_{j \to i}^t} \left( \,\|w_i + th_i - x\|^2 - \|w_j + th_j -x\|^2\,\right)\, p(x)\,dx,
\end{equation}
where the domain of the integral $V_{j\to i}^t = \big(V_i^t \setminus V_i\big) \cap \big(V_j \setminus V_j^t\big)$ is the set of points $x \in \mathbb{R}^d$ that originally began in the $j$th partition $V_j(w)$ but after the small perturbation, ended up in the $i$th partition $V_i(w + th)$. Indeed, $V_i^t \setminus V_i$ is the disjoint union $\bigcup_{j} V_{j\to i}^t$, since every point in $V_i^t$ not originally in $V_i$ must have come from some other partition $V_j$. Similarly, $V_j \setminus V_j^t = \bigcup_i V_{j\to i}^t$. 

Intuitively, if $x$ swaps partitions due to a small perturbation, then $\|w_i - x\|^2 \sim \|w_j - x\|^2$. In fact, we will show that the magnitude of $\|w_i + th_i - x\|^2 - \|w_j + th_j - x\|^2$  is $O(t\|x\|)$ for $x \in V_{j\to i}^t$. First, to simplify notation, let:
\[\alpha = \frac{w_i - w_j}{2} \quad \beta = \frac{w_i + w_j}{2}\quad \gamma = \frac{h_i - h_j}{2} \quad \delta = \frac{h_i + h_j}{2}.\]
By the polarization identity, we have:
\begin{align*} 
\|w_i - x\|^2 - \|w_j - x\|^2\quad  &= \quad 4 \big\langle \alpha, \beta - x\big\rangle \\ 
\|w_i + th_i - x\|^2 - \|w_j + th_j - x\|^2\quad &=\quad 4 \big\langle \alpha + t \gamma, \beta + t\delta - x\big\rangle.
\end{align*}
Furthermore, because $x \in V_{j\to i}^t$, we have the inequalities:
\begin{align*}
    \|w_j - x\|^2 \leq \|w_i - x\|^2\quad \quad \textrm{and}\quad\quad \|w_i + th_i - x\|^2 \leq \|w_j + th_j -x\|^2.
\end{align*}
Plugging in the equality from polarization, we obtain:
\begin{align*}
    \underbrace{\|w_i + th_i - x\|^2 - \|w_j + th_j - x\|^2}_{\leq 0}
    &= \underbrace{\|w_i - x\|^2 - \|w_j - x\|^2}_{\geq 0} \, + \, 4t \big(\langle \alpha ,\delta\rangle + \langle \gamma , \beta - x\rangle + t \langle \gamma, \delta\rangle\big),
\end{align*}
which implies that $\left|\|w_i + th_i - x\|^2 - \|w_j + th_j - x\|^2\right| \leq Ct( \|x\| + 1)$ when $t < 1$, where $C$ is a constant that may depend on $w$. It follows that:
\[\left|\int_{V_{j\to i}^t} \left( \,\|w_i + th_i - x\|^2 - \|w_j + th_j -x\|^2\,\right)\, p(x)\,dx\right| \leq \int_{V_{j\to i}^t} Ct( \|x\| + 1) \, p(x)\,dx.\]
As $\displaystyle \E_p\left[\|X\|^2\right] <\infty$, we know $p$ also has bounded first moment. By dominated convergence:
\[\lim_{t\to 0} \frac{1}{t} \int_{V_{j\to i}^t} Ct( \|x\| + 1) \, p(x) \, dx = 0,\]
since $V_{j\to i}^t$ decreases to some measure zero subset of $V_{i} \cap V_{j}$. And so, (\ref{eqn:perturbation-term}) is $o(t)$. 
\end{proof}

\subsubsection{A short proof using Leibniz integral rule}
Suppose we are given an integral where both its domain and integrand are time-varying. Then the derivative of that integral is given by Leibniz rule:
\[\frac{d}{dt} \int_{a(t)}^{b(t)} F(x,t)\,dx = \int_{a(t)}^{b(t)} \frac{\partial F(x,t)}{\partial t} \, dx +  \bigg(b'(t) F\big(b(t), t\big) - a'(t) F\big(a(t), t\big)\bigg).\]
In particular, we break down the time derivative into two pieces: (i) the accumulated time derivative at each point in the domain, and (ii) the weighted velocity at the boundary at which the domain is expanding or contracting. In higher dimensions, the time derivative of a volume integral can be decomposed into the same two pieces. 

But generalizing Leibniz rule to higher dimensions, we need some notation. Let $\Omega(t) \subset \mathbb{R}^n$ be a smoothly time-varying differentiable $n$-manifold for $t \in (-\tau,\tau)$ with boundary $S(t) = \partial \Omega(t)$. That is, there is a domain $U \subset \mathbb{R}^n$ and a continuously differentiable map $\phi: (-\tau, \tau) \times U \to \mathbb{R}^n$ where $\phi_t$ is a diffeomorphism of $U$ onto $\Omega(t)$. 

We write $x = x(t,u) = \phi(t, u)$ and $v = \partial x / \partial t$. For points $x \in S(t)$ on the boundary, denote the surface normal by $N = N(x)$. The \emph{surface velocity} at $x \in S(t)$ is defined as $C(x) = N^\top v$, which is an invariant in the sense that it is coordinate-independent \citep{grinfeld2013introduction}.

\begin{theorem}[General Leibniz rule, \cite{grinfeld2013introduction}]
Let $\Omega(t) \subset \mathbb{R}^n$ be a smoothly time-varying smooth $n$-manifold with boundary $S(t)$ over times $t \in (-\tau, \tau)$. Let $C(t,x)$ be the surface velocity of the point $x \in S(t)$ at time $t$. If $F: (-\tau, \tau) \times \mathbb{R}^n \to \mathbb{R}^m$ is smooth, then for $t \in (-\tau, \tau)$:
\[\frac{d}{dt} \int_{\Omega(t)} F \, d\Omega = \int_{\Omega(t)} \frac{\partial F}{\partial t} \, d\Omega + \int_{S(t)} F\, C dS.\]
\end{theorem}

As a result, if we consider the directional derivative of our objective $f$ in the direction of $h$,
\begin{equation} \label{eqn:f-directional-derivative}
\frac{d}{dt} f(w + th) = \frac{1}{2} \sum_{i \in [k]} \frac{d}{dt} \int_{V_i(w + th)} \|w_i + th_i - x\|^2 p(x) \, dx,
\end{equation}
each of the integrals will split into two: (i) the integrals computing the accumulated rate of change, and (ii) those computing weighted surface velocities at the boundaries of the Voronoi partition. The first exactly coincides with $\frac{d}{dt}\, g(w+th; w)$. The second terms vanish since the weighted surface velocities at the boundaries of two partitions exactly cancel each other out. Formally, we have:

\begin{proof}[Proof of Lemma~\ref{lem:k-means-grad}]
Fix $w, h \in \mathbb{R}^{k \times d}$ where $h$ is unit. The directional derivative $D_hf(w)$ is given by (\ref{eqn:f-directional-derivative}) evaluated at time $t = 0$. Applying the general Leibniz rule yields:
\begin{equation} \label{eqn:leibniz-k-means}
    D_hf(w) = \sum_{i \in [k]} \,h_i^\top \int_{V_i(w)} (w_i - x) p(x)\, dx + \frac{1}{2} \sum_{i \in [k]} \int_{\partial V_i(w)} \|w_i - x\|^2 p(x) \, C_idS,
\end{equation}
where $C_i(x)$ is the surface velocity of a point $x \in \partial V_i(w)$ at time $0$. Notice that if $x \in \partial V_i(w)$, then it is also contained in exactly one other boundary, $x \in \partial V_j(w)$. 
On the one hand, the weight of the integrands are equal $\|w_i - x\|^2 p(x) = \|w_j - x\|^2 p(x)$. But on the other, the surface velocities of $x$ of $\partial V_i \cap \partial V_j$ are equal and opposite, $C_i(x) = - C_j(x)$, since it is an invariant. Therefore,
\[\int_{\partial_iV(w) \cap \partial_jV(w)} \big(\|w_i - x\|^2p(x) \, C_i\big) + \big(\|w_j - x\|^2p(x)\, C_j\big)\, dS = 0.\]
Thus, the second set of integrals in (\ref{eqn:leibniz-k-means}) vanishes. By the chain rule, $D_hf(w) = h^\top \nabla f(w)$, so:
\[\nabla_{w_i}\, f(w) = \int_{V_i(w)} (w_i - x)p(x)\, dx.\qedhere\]
\end{proof}

\subsection{An analytic upper bound of the cost}
Since $g(\,\cdot\,;w')$ is quadratic, we can give the upper bound analytically using our computation of $\nabla f$.

\begin{lemma}[Quadratic upper bound] \label{lemma:quad-upper-bound}
Let $p$ a density on $\mathbb{R}^d$ have bounded second moment. If $f$ is the $k$-means objective (\ref{eqn:kmeans-cost}), then for all $w, w^+ \in \mathbb{R}^{k\times d}$
    \begin{align}\label{eqn:quad-upper-bounds-general} 
        f(w^+) &\leq f(w) + \big\langle \nabla f(w), w^+ - w\big\rangle + \frac{1}{2} (w^+ - w)^\top \mathbf{H} \,(w^+ - w),
    \end{align}
where we let $\mathbf{H} = \mathbf{H}_w\,g(w;w)$ be the Hessian of $g(\,\cdot\,;w)$. In particular, \begin{align}\label{eqn:quad-upper-bounds} 
    f(w^+) &\leq f(w) + \big\langle \nabla f(w), w^+ - w\big\rangle + \frac{1}{2} \|w^+ - w\|^2.
\end{align}
\end{lemma}
\begin{proof}
Let $w, w^+ \in \mathbb{R}^{k\times d}$. Recall that $f(w^+)$ is upper bounded by $g(w^+;w)$ by Proposition~\ref{prop:surrogate}. Because $g(\,\cdot\,;w)$ is quadratic, it is equal to its second-order Taylor expansion. Then, we have:
\begin{align*}
    f(w^+) &\leq g(w^+;w) \notag
    = g(w;w) + \big\langle \nabla_w g(w;w), w^+ - w\big\rangle + \frac{1}{2}  (w^+ - w)^\top \mathbf{H}\, (w^+-w).
\end{align*}
The first assertion (\ref{eqn:quad-upper-bounds-general}) follows because $f(w) = g(w;w)$ and $\nabla f(w) = \nabla_w\, g(w;w)$.

Notice that the Hessian $\mathbf{H}$ is constant since $g$ is quadratic:
\begin{equation} \label{eqn:hess-g}
\mathbf{H}_w\, g(w^+;w)\,_{ij} = 
\begin{cases}
\ P_i(w)\, \mathbf{I}_{d\times d} & i = j \\
\ \phantom{P_i(w} 0 & i \ne j,
\end{cases}
\end{equation}
where $\mathbf{I}_{d\times d}$ is the $d$-dimensional identity matrix. Because $P(w)$ is a probability vector, the spectral norm is bounded, $\displaystyle \|\mathbf{H}_{w^+}\, g(w^+;w)\|_* = \max_{i \in [k]} \, P_i(w) \leq 1$. From this, (\ref{eqn:quad-upper-bounds}) immediately follows.
\end{proof}
\section{Analysis of online \textit{k}-means algorithms}
\subsection{Convergence of cost} \label{sec:cost-conv}
\begin{lemma} \label{lem:supermartingale-approx}
Let $f$ be the $k$-means cost function and let $(W^{(n)}, H^{(n)}, X^{(n)})_{t=1}^\infty$ be generated by the \ref{alg:generalized-online-kmeans} algorithm. Then:
\begin{equation} \label{eqn:approx-supermartingale}
f(W^{(n+1)}) \leq f(W^{(n)}) - A_{n+1} + N_{n+1},
\end{equation}
where $A_{n+1}$ is the exact gradient descent term and $N_{n+1} = -B_{n+1} + C_{n+1}$ is the noise term:
\begin{itemize}
    \item $\displaystyle A_{n+1} = \sum_{i\in [k]} H_i^{(n+1)} P_i^{-1}(W^{(n)}) \|\nabla_{w_i} f(W^{(n)})\|^2$
    \item $\displaystyle B_{n+1} = \sum_{i\in [k]} H_i^{(n+1)} \nabla_{w_i}f(W^{(n)})^\top \big(M_i(W^{(n)}) - X_i^{(n+1)}\big)$
    \item $\displaystyle C_{n+1} = \frac{1}{2} \sum_{i\in[k]} \big(H_i^{(n+1)}\big)^2 \|W_i^{(n)} - X_i^{(n+1)}\|^2$.
\end{itemize}
In particular, $(A_n)_{n=1}^\infty$ and $(C_n)_{n=1}^\infty$ are nonnegative sequences; and, since $H^{(n+1)}$ and $X^{(n+1)}$ are conditionally independent given $\mathcal{F}_n$, $(B_n)_{n=1}^\infty$ is a martingale difference sequence.
\end{lemma}
\begin{proof}
We simply rewrite the update in the following form:
\begin{align} 
W^{(n+1)}_i &= W^{(n)}_i - H^{(n+1)}_i \left(\textcolor{blue}{W^{(n)}_i - M_i(W^{(n)})} + \textcolor{orange}{M_i(W^{(n)}) - X_i^{(n+1)}}\right) \notag
\\&= W^{(n)}_i - H^{(n+1)}_i \textcolor{blue}{P_i(W^{(n)})^{-1} \nabla_{w_i}f(W^{(n)})}  - H^{(n+1)}_i \left(\textcolor{orange}{M_i(W^{(n)}) - X_i^{(n+1)}}\right), \label{eqn:update-step}
\end{align}
where the second line follows from \Cref{lem:k-means-grad} showing $\nabla_{w_i}f(w) = P_i(w) \big(w_i - M_i(w)\big)$. Recall the quadratic upper bound in \Cref{lemma:quad-upper-bound}, reproduced here:
\begin{equation}
    f(W^{(n+1)}) \leq f(W^{(n)}) + \langle \nabla f(W^{(n)}), W^{(n+1)} - W^{(n)}\rangle + \frac{1}{2} \|W^{(n+1)} - W^{(n)}\|^2 \tag{\ref{eqn:quad-upper-bounds}}
\end{equation}
Combining (\ref{eqn:quad-upper-bounds}) and  (\ref{eqn:update-step}) immediately yields (\ref{eqn:approx-supermartingale}).
\end{proof}

\begin{lemma} \label{lemma:convergence-n-b}
Let $(N_n)_{n=1}^\infty$ as in \Cref{lem:supermartingale-approx}. Suppose that:
\[\sum_{n=1}^\infty \sum_{i\in[k]} \big(H_i^{(n)}\big)^2 < \infty\quad \textrm{a.s.}\]
Then the series $\displaystyle \sum_{n=1}^\infty N_n = - \sum_{n=1}^\infty B_n + \sum_{n=1}^\infty C_n < \infty$ converges almost surely.
\end{lemma}
\begin{proof}
\Cref{assume:finite} and \ref{assume:lr} imply that the all iterates $W^{(n)}_i$ and updates ${X_i}^{(n+1)}$ remain in the closed ball $B(0,R)$. Recall from \Cref{lem:k-means-grad} that $\nabla_{w_i} f(w) = P_i(w) \cdot \big(w - M_i(w)\big)$. Thus:
\[\big|\nabla_{w_i} f(W^{(n)})^\top \big(M_i(W^{(n)}) - X_i^{(n+1)}\big)\big| < 4R^2 \quad \textrm{and}\quad \|W_i^{(n)} - X_i^{(n+1)}\|^2 < 4R^2.\]
The series $\sum B_n$ converges almost surely by martingale convergence, \Cref{thm:martingale-convergence}, which we may apply because we have:
\[\sum_{n=1}^\infty \E[B_{n}^2] \leq 16R^4 \cdot \sum_{n=1}^\infty \sum_{i \in [k]}  \big(H_i^{(n)}\big)^2 < \infty.\]
The series $\sum C_n$ converges almost surely since it is dominated by a convergent series:
\[\sum_{n=1}^\infty C_n \leq 2R^2\cdot \sum_{n=1}^\infty \sum_{i \in [k]} \big(H_i^{(n)}\big)^2 < \infty. \qedhere\]
\end{proof}

First we show that the cost $f(W^{(n)})$ converges. Notice that if the noise term $N_{n+1}$ in \Cref{lem:supermartingale-approx} did not contain the nonnegative $C_{n+1}$ term, then $f(W^{(n)})$ is seen to be a supermartingale bounded below since $f \geq 0$. Then, the convergence of $f(W^{(n)})$ would immediately follow from the martingale convergence theorem. Even though $f(W^{(n)})$ is not a supermartingale, we can obtain convergence since $\sum C_n < \infty$ converges almost surely, from \Cref{lemma:convergence-n-b}. This next lemma proves this formally.

\propcostconv*

\begin{proof}
Let $(M_n)_{n=1}^\infty$ be defined by:
\[M_{n+1} = f(W^{(n+1)}) - \sum_{\tau=0}^n C_{\tau+1}.\]
\Cref{lem:supermartingale-approx} shows that $M_{n}$ is an $\mathcal{F}_n$-supermartingale:
\begin{align*}
\E\big[M_{n+1}\, \big|\, \mathcal{F}_n\big] &\leq f(W^{(n)}) - \E\big[A_{n+1} + B_{n+1} - C_{n+1} \,\big|\, \mathcal{F}_n\big] - \E\left[\sum_{\tau = 0}^{n-1} C_{\tau + 1} + C_{n+1}\,\middle|\, \mathcal{F}_n\right]
\\&= f(W^{(n)}) - \sum_{\tau = 0}^{n-1} C_{\tau + 1} - \E[A_{n+1}\,|\,\mathcal{F}_n] \leq M_n,
\end{align*}
where we used the fact that $(A_n)_{n=1}^\infty$ is nonnegative and $(B_n)_{t=0}^\infty$ is an $\mathcal{F}_n$-martingale. Furthermore, because $(C_n)_{n=1}^\infty$ is nonnegative and \Cref{lemma:convergence-n-b} shows that $\sum C_n < \infty$ converges, the supermartingale is bounded below:
\[-\infty < - \sum_{t=1}^\infty C_n \leq M_n.\]
By the martingale convergence theorem, both $(M_n)_{n=1}^\infty$ and $f(W^{(n)})$ converge almost surely.
\end{proof}

\begin{theorem}[Martingale convergence, \cite{durrett2019probability}]
\label{thm:martingale-convergence}
Let $(M_n)_{n \in \mathbb{N}}$ be a (sub)martingale with: 
\[\sup_{n \in \mathbb{N}} \E\big[M_n^+\big] < \infty,\]
where $M_n^+ := \max\{0, M_n\}$. Then as $n \to \infty$, $M_n$ converges a.s. to a limit $M$ with $\E[|M|]< \infty$.
\end{theorem}
\begin{remark}[Specific forms of martingale convergence] We use two specific forms of \Cref{thm:martingale-convergence} in our proofs. The first applies to martingale difference sequences $(B_n)_{n\in \mathbb{N}}$. Let $M_n = \sum_{m=1}^n B_n$. Since the terms in a martingale difference sequence are orthogonal, we have for all $n \in \mathbb{N}$:
\[\E\left[M_n^+\right]^2 \leq \E\left[\bigg(\sum_{m=1}^n B_m\bigg)^2\right] = \sum_{m=1}^n \E\left[B_m^2\right].\]
It follows that the condition $\displaystyle \sum_{n=1}^\infty \E[B_n^2] < \infty$ implies $\displaystyle \sup_{n \in \mathbb{N}} \E\left[M_n^+\right] < \infty$.

The other applies to lower bounded supermartingales $(M_n)_{n \in \mathbb{N}}$. As $(-M_n)_{n \in \mathbb{N}}$ is then an upper bounded submartingale, it converges to some $-M$. We may apply martingale convergence: let $c \in \mathbb{R}$ be a lower bound such that $M_n > c$ almost surely. Then $-M_n$ is a submartingale with: 
\[\displaystyle \sup_{n \in \mathbb{N}} \, \E\left[-M_n^+\right] \leq \max\{0, -c\} < \infty.\]
\end{remark}

\subsection{Convergence of iterates} \label{sec:conv-iter}
\lemnondegenerate*

\begin{proof}
By assumption, $W^{(0)} \in \mathcal{D}$. It suffices to show by induction that $W^{(n+1)}$ is non-degenerate almost surely if $W^{(n)}$ is non-degenerate. Note that ${W_i}^{(n+1)} \in V_i(W^{(n)})$ as it is a convex combination of points in the Voronoi region $V_i(W^{(n)})$. Therefore, the only way for two initially distinct centers ${W_i}^{(n+1)}$ and ${W_j}^{(n+1)}$ to possibly meet is if the updates ${X_i}^{(n+1)}$ and ${X_j}^{(n+1)}$ both come from the boundary of their Voronoi cells. But the boundary has measure zero; this occurs almost never.
\end{proof}

\lemlrbounditers*
\begin{proof} 
We claim that $W_i^{(n)} \in B(0,R)$ for all $n \in \mathbb{N}_0$ by induction. \Cref{assume:finite} states that $p$ is supported only in $B(0,R)$. Since $W^{(0)}_i$ comes from the support of $p$, this claim holds for $n = 0$. If $W_i^{(n)} \in B(0,R)$, then $W_i^{(n+1)}$ is a convex combination of points in $B(0,R)$ almost surely:
\[W_i^{(n+1)} = \big(1 - H_i^{(n+1)}\big) \cdot W_i^{(n)} + H_i^{(n+1)} \cdot X_i^{(n+1)},\]
since $H_i^{(n+1)} \in [0,1]$ and $X_i^{(n+1)} \sim p$.

As a result of this, we can upper bound the displacement:
\begin{align}
    \|W^{(m)} - W^{(n)}\| &\overset{(i)}{\leq} \sum_{j \in [k]} \big\|W_j^{(m)} - W_j^{(n)}\big\| \notag
    \\&\overset{(ii)}{\leq}  \sum_{j \in [k]} \sum_{m \leq n' < n} \big\| H_j^{(n'+1)} \cdot \big(W_i^{(n')} - X_i^{(n'+1)}\big)\big\| \notag
    \\&\overset{(iii)}{\leq}  2R \cdot \sum_{j \in [k]} \sum_{m \leq n' < n} H_j^{(n'+1)}, \label{eqn:learning-bound-to-iterate-bound}
\end{align}
where (i) follows from Minkowski's inequality, (ii) follows from triangle inequality, and (iii) follows from our initial claim since $W_i^{(n')}, X_i^{(n'+1)} \in B(0,R)$ almost surely, so $\|W_i^{(n')} - X_i^{(n'+1)}\| < 2R$.
\end{proof}

\begin{theorem}[Convergence of iterates, generalized] \label{thm:iterates-conv}
Let $(W^{(n)})_{n=0}^\infty$ and $H^{(n+1)}$ be as in \Cref{prop:cost-conv}. Suppose there exists $\epsilon_0 > 0$ such that $\{\|\nabla f\| \leq \epsilon_0\}$ is compact, and for all $i \in [k]$,
\begin{enumerate}
    \item[$(\mathrm{A}1)$] For any $\epsilon \in (0, \epsilon_0)$, there is an $r_0 \equiv r_0(\epsilon) > 0$ so that if $r \in (0,r_0)$, then there exist $T : \mathbb{N} \to \mathbb{N}$, $m_0 \in \mathbb{N}$, and $s, c > 0$, which may all depend on $\epsilon$ and $r$, so that:
\[\Pr\left(\sum_{j \in [k]}\sum_{m \leq n < T(m)}  H_j^{(n+1)}  < r \ \ \mathrm{and}\!\! \sum_{m \leq n < T(m)} H_i^{(n+1)} > s \, \middle|\, \mathcal{F}_m, \|\nabla_{w_i} f(W^{(m)})\| \in [\epsilon,\epsilon_0) \right) > c,\]
for any $m > m_0$, and,
    \item[$(\mathrm{A}2)$] If $\displaystyle \liminf_{n \to \infty} \, \|\nabla_{w_i} f(W^{(n)})\| \geq \epsilon_0$, then $\displaystyle \sum_{n \in \mathbb{N}} H_i^{(n)} = \infty$ almost surely.
\end{enumerate}
Then, the iterates $W^{(n)}$ of \ref{alg:generalized-online-kmeans} asymptotically converge to stationary points of $f$ almost surely, where $f$ is the $k$-means cost (\ref{eqn:kmeans-cost}).
\end{theorem}

\begin{proof}
We claim that for all $\epsilon > 0$, the iterates $W^{(n)}$ eventually never return to the set $\{\|\nabla f\| > \epsilon\}$ almost surely. If so, then the sequence of gradients converges to zero $\|\nabla f(W^{(n)})\| \to 0$ almost surely, since the claim holds simultaneously for any countable sequence of $\epsilon_j \downarrow 0$. Because the map $w \mapsto \|\nabla f(w)\|$ is continuous and the iterates eventually remain in a compact region $\{\|\nabla f\| \leq \epsilon_0\}$, the convergence of gradients $\|\nabla f(W^{(n)})\|$ to zero implies the almost sure convergence of the iterates $W^{(n)}$ to the set of stationary points (\Cref{lem:top-conv}):
\[\limsup_{n\to\infty}\, \inf_{\{w : \nabla f(w) = 0\}}\, \|W^{(n)} - w\| = 0 \quad \mathrm{a.s.}\]

The claim remains: iterates $W^{(n)}$ eventually never return to the set $\{\|\nabla f\| > \epsilon\}$ for all $\epsilon > 0$. Note that $\|\nabla f(w)\|$ is upper bounded by $\sum_{i \in [k]} \|\nabla_{w_i} f(w)\|$, so it suffices to consider each center individually and show that the iterates eventually never return to $\{\|\nabla_{w_i} f\| > \frac{\epsilon}{k}\}$. And so, we show that for all $i \in [k]$ and $\epsilon > 0$, if $\|\nabla_{w_i} f(W^{(n)})\| > \epsilon$ infinitely often, then $f(W^{(n)})$ does not converge. As this would contradict \Cref{prop:cost-conv}, we have $\|\nabla_{w_i} f(W^{(n)})\| > \epsilon$ finitely often almost surely.

Consider the case $\epsilon < \epsilon_0$. We show that iterates eventually never return to  $\{\|\nabla_{w_i} f\| \in [\epsilon, \epsilon_0)\}$. Fix $i \in [k]$ and any $\epsilon \in (0,\epsilon_0)$. Note that $\{\|\nabla_{w_i} f\| \leq \frac{\epsilon}{2}\}$ and $\{\|\nabla_{w_i} f\| \in [\epsilon, \epsilon_0]\}$ are disjoint compact sets; because $w \mapsto \|\nabla f(w)\|$ is continuous, they are closed subsets of the compact set $\{\|\nabla f\| \leq \epsilon_0\}$. And so, these two sets are separated by some distance $R_0 > 0$. Without loss of generality, we may assume that $r_0$ given in (A1) satisfies $r_0 \leq R_0/2R$. Fix any $r \in (0,r_0)$.

We may now apply condition (A1) to control the behavior of the iterates for a non-negligible number of iterations upon entering the set $\{\|\nabla_{w_i}f \| \in [\epsilon,\epsilon_0)\}$; given $(\epsilon,r)$, let $(T,m_0,s,c)$ be chosen so that (A1) holds. Suppose at time $m > m_0$, the iterate $W^{(m)}$ enters this set. For parsimony, call the two events within the first probability in the theorem statement $\Xi_1$ and $\Xi_2$,
\[\Xi_1 = \left\{\sum_{j \in[k]} \sum_{m \leq n < T(m)} H_j^{(n+1)} < r\right\}\qquad\textrm{and}\qquad \Xi_2 = \left\{\sum_{m\leq n < T(m)} H_i^{(n+1)} > s\right\}.\]
We show that given $m$ sufficiently large, if both of these events hold, then iterates remain in the set $\{\|\nabla_{w_i} f\| > \frac{\epsilon}{2}\}$ for a sufficient amount of time to decrease $f(W^{(n)})$ by a constant amount. This allows us to apply Borel-Cantelli to show that $f(W^{(n)})$ does not converge. For the first claim, recall that \Cref{lem:lr-iterates-bound} bounds the distance iterates travel away from $W^{(m)}$ via the summed learning rates:
\[ \sum_{j \in [k]}\sum_{m \leq n < T(m)} H_j^{(n+1)} < r \qquad \implies \qquad \sup_{m \leq n < T(m)} \|W^{(m)} - W^{(n)}\| < 2R\cdot r < R_0.\]
Consequently, $\Xi_1$ implies that $\|\nabla_{w_i} f(W^{(n)})\| > \frac{\epsilon}{2}$ on the interval $m \leq n < T(m)$. The second event $\Xi_2$ can be used to show that $f(W^{(n)})$ decreases by at least a constant amount on this interval. \Cref{lem:supermartingale-approx} shows for all $n \in \mathbb{N}$,
\[f\big(W^{(n)}\big) \leq f\big(W^{(m)}\big) - \sum_{m \leq n' < n}  A_{n' + 1} + \sum_{m \leq n' < n} N_{n' +1}.\]
\Cref{lemma:convergence-n-b} shows that $\sum_{n=0}^\infty N_{n+1}$ converges almost surely; thus, $\sum_{m \leq n' <n} N_{n'+1} < \delta$ when $m$ is sufficiently large. That is, for any $\delta > 0$, there almost surely exists an $\mathbb{N}$-random variable $M_\delta$ so:
\[\bigg|\sum_{n' \geq M_\delta} N_{n'+1}\bigg| < \delta/2.\] 
In particular, $\sum_{m \leq n' <n} N_{n'+1} < \delta$ holds for all $m > M_\delta$. Let $m' = T(m) - 1$. Then:
\begin{align*}
    f\big(W^{(m')}\big) &\overset{(i)}{\leq} f\big(W^{(m)}\big) - \sum_{m \leq n < m'} \sum_{j\in [k]} H_j^{(n+1)} P_j^{-1}(W^{(n)}) \|\nabla_{w_j} f(W^{(n)})\|^2 + \delta
    \\&\overset{(ii)}{\leq} f(W^{(m)}) - \frac{\epsilon^2}{4} \sum_{m \leq n < m'} H_{i}^{(n+1)} + \delta
    \\&\overset{(iii)}{\leq} f(W^{(m)}) - \frac{s\epsilon^2}{4} + \delta
    \\&\overset{(iv)}{<} f(W^{(m)}) - \delta,
\end{align*}
where (i) substitutes in the expression for $A_{n + 1}$, (ii) drops the summation over centers $j \ne i$, and $\|\nabla_{w_i}f(W^{(n)})\| >\frac{\epsilon}{2}$  holds if the first event occurs, (iii) follows if $\Xi_2$ occurs, and (iv) sets $\delta < s\epsilon^2/ 8$.

We have thus shown that if condition (A1) holds, $m > \max\{m_0, M_\delta\}$, and $\delta < s\epsilon^2/8$, then:
\begin{equation}  \label{eqn:decrease-f}
\Pr\left(f(W^{(n)}) < f(W^{(m)}) - \delta \textrm{ for some }m \leq n < T(m)\, \bigg|\, \mathcal{F}_m, \|\nabla_{w_i} f(W^{(m)}) \| \in [\epsilon,\epsilon_0)\right) > c.
\end{equation}
That is, if $\|\nabla_{w_i} f(W^{(m)})\|$ is large at iteration $m$, then with positive probability, within a bounded amount of time, $f(W^{(n)})$ will decrease by a constant amount $\delta$. But we also know that $f(W^{(n)})$ converges almost surely to some $f^*$, by \Cref{prop:cost-conv}, and so decreases by $\delta$ only a finite number of times. We claim that by Borel-Cantelli, \Cref{lem:borel-cantelli}, the event $\|\nabla_{w_i} f(W^{(n)})\| \in [\epsilon,\epsilon_0)$ must also occur only finitely often. 

Assume for the sake of contradiction that $\|\nabla_{w_i} f(W^{(n)})\| \in [\epsilon,\epsilon_0)$ infinitely often. Then we can define the infinite sequence of stopping times:
\[\tau_0 = 0 \qquad \textrm{and}\qquad \tau_{j+1} = \inf \{n > \tau_{j} + T(\tau_{j}) : \|\nabla_{w_i} f(W^{(n)})\| \in [\epsilon,\epsilon_0)\}.\]
Then (\ref{eqn:decrease-f}) states that when $\tau_j > \max\{m_0, M\}$,
\[\Pr\left(f(W^{(n)}) < f(W^{(\tau_{j})}) - \delta \textrm{ for some } \tau_j \leq n < T(\tau_j) \,\middle|\, \mathcal{F}_{\tau_{j}}\right) > c',\]
where the event in the probability is $\mathcal{F}_{\tau_{j+1}}$-measurable. Borel-Cantelli, \Cref{lem:borel-cantelli}, implies that $f(W^{(n)})$ decreases by a constant amount $\delta$ infinitely often, contradicting the convergence of $f(W^{(n)})$.

To finish the proof, we need to show that the iterates eventually never return to the set $\{\|\nabla_{w_i} f\| \geq \epsilon_0\}$. We do this by ruling out (i) after some iteration $m$, the iterates never leave this set, and (ii) the iterates exit and re-enter this set infinitely often. The first case is impossible, for then condition (A2) implies that $\sum_{n \in \mathbb{N}} H_i^{(n)} = \infty$ almost surely. By \Cref{lem:supermartingale-approx}, this leads to an unbounded decrease in cost, 
\begin{align*}
    \liminf_{N \to \infty} \, f(W^{(N)}) &\leq \lim_{N \to \infty}\, \left(f(W^{(m)}) - \epsilon_0^2 \sum_{m \leq n < N} H_i^{(n+1)} + \delta\right) = - \infty.
\end{align*}
The second case is also impossible; when the learning rates become sufficiently small, each time the iterates leave $\{\|\nabla_{w_i} f\| \geq \epsilon_0\}$, they must  enter $\{\|\nabla_{w_i} f\| \in [\epsilon, \epsilon_0)\}$. Thus, the iterates eventually never return to $\{\|\nabla_{w_i} f\| \geq \epsilon_0\}$. This shows that for all $\epsilon > 0$, we almost surely have $\|\nabla_{w_i} f(W^{(n)})\| > \epsilon$ finitely often.
\end{proof}

We now prove \Cref{thm:iterates-conv-simple}, the simplified version of \Cref{thm:iterates-conv} seen \Cref{sec:iter-conv-sketch} (reproduced below the next lemma). While simpler, it imposes a stronger condition on the learning rate:

\begin{lemma} \label{lem:inf-learning}
Let $i \in [k]$ and $\epsilon > 0$. Suppose there exists $T: \mathbb{N} \to \mathbb{N}$, $m_0 \in \mathbb{N}$, and $s,c > 0$,
\[\Pr\left(\sum_{m \leq n < T(m)} H_i^{(n+1)} > s\,\middle|\, \mathcal{F}_m, \|\nabla_{w_i} f(W^{(m)})\| > \epsilon\right) > c,\]
for all $m > m_0$. Then $\displaystyle \liminf_{n \to \infty} \|\nabla_{w_i} f(W^{(n)})\| \geq \epsilon$ implies $\displaystyle \sum_{n \in \mathbb{N}} H_i^{(n)} = \infty$ almost surely.
\end{lemma}

\begin{proof}
Suppose that there is an $\mathbb{N}$-random variable $M$ such that if $m > M$, then $\|\nabla_{w_i} f(W^{(m)})\| \geq \epsilon$. That is, the limit infimum condition holds. By assumption, we have:
\[\Pr\left(\sum_{m\leq n < T(m)} H_i^{(n+1)} > s \,\middle|\, \mathcal{F}_m, m > \max\{m_0, M\}\right) > c.\]
The Borel-Cantelli lemma (\Cref{lem:borel-cantelli}) shows that there are infinitely many (non-overlapping) intervals $m_j \leq n < T_r(m_j)$ on which the sum of $H_i^{(n+1)}$ is at least $s$, and so the total sum is infinite almost surely.
\end{proof}

\lemcompacteps*

\begin{proof}
Because the inclusion map $\iota:\mathcal{D}_R \to \mathbb{R}^{k\times d}$ is continuous, if $\{\nabla f = 0\}$ is compact in $\mathcal{D}_R$, then it is compact in $\mathbb{R}^{k\times d}$. On the other hand, the set of degenerate points $Z := \mathbb{R}^{k\times d} \setminus \mathcal{D}$ is closed in $\mathbb{R}^{k\times d}$, for it is the union of closed sets $A_{ij}$ for $i\ne j$ defined by:
\[A_{ij} := \{\|w_i - w_j\| = 0\}.\]
If a closed set and a compact set in a metric space are disjoint, then they are separated by some positive distance $\alpha > 0$; $\{\nabla f = 0\}$ and $Z$ are disjoint, so no limit point of $\{\nabla f= 0\}$ is degenerate. 

And, because $\nabla f$ is continuous on $\mathcal{D}$, this implies that that there exists $\epsilon_0 > 0$ such that $\{\|\nabla f\| \leq \epsilon_0\}$ is compact. In particular, the $\alpha/2$-expansion of $\{\nabla f = 0\}$ is compact in $\mathcal{D}$, where the $\alpha/2$-expansion is the set of points a distance less than or equal to $\alpha/2$ from a stationary point. Additionally, its boundary is compact and separated from $\{\nabla f = 0\}$, so $w \mapsto \|\nabla f(w)\|$ attains a minimum $2\epsilon_0 > 0$ on it. It follows by continuity that $\{\|\nabla f \| \leq \epsilon\}$ for any $\epsilon \in [0,\epsilon_0]$ is a closed set contained in the $\alpha/2$-expansion, hence compact.
\end{proof}

The following lemma is used later in \Cref{sec:lloyd}, using the same argument to show compactness:

\begin{lemma} \label{lem:non-empty-level-set}
Let $\epsilon_0 > 0$ be given so that the set $\{\|\nabla_{w_i} f\| \leq \epsilon_0\}$ is compact. Fix  $0 \leq \epsilon \leq \epsilon' \leq \epsilon_0$. Then the level set $K := \{\|\nabla_{w_i} f\| \in [\epsilon, \epsilon']\}$ is a nonempty compact set.
\end{lemma}

\begin{proof}
By \Cref{lem:k-means-grad}, the map $\phi: w \mapsto \|\nabla_{w_i} f(w)\|$ is continuous. Since $K = \phi^{-1}([\epsilon, \epsilon'])$ is the  inverse of a closed set, it is a closed subset of $\{\|\nabla_{w_i} f\| \leq \epsilon_0\}$, hence compact. Furthermore, $K$ is nonempty; if this were not the case, then we claim that $\{\|\nabla_{w_i} f\| \leq \epsilon\} = \mathcal{D}_R$. Assuming the claim for now, we arrive at a contradiction since $\mathcal{D}_R$ is not compact. 

For the claim, note that $w \mapsto \|\nabla_{w_i} f(w)\|$ is continuous and that the set $\{\|\nabla_{w_i} f\| = 0\}$ is nonempty. So if there were some $w \in \mathcal{D}_R$ with $\|\nabla_{w_i}f(w)\| > \epsilon$, then the intermediate value theorem implies that there is some other point $w'$ with $\|\nabla_{w_i} f(w')\| = \epsilon$, which violates our assumption.
\end{proof}

\thmiteratesconv*

\begin{proof}
By \Cref{lem:compact-eps}, there exists $\epsilon_0$ such that $\{\|\nabla_{w_i} f\| \leq \epsilon_0\}$ is compact for all $i \in [k]$. The conclusion follows from verifying the conditions of \Cref{thm:iterates-conv}. Condition (A1) is assumed. Condition (A2) follows from \Cref{lem:inf-learning}, in which we set the $\epsilon$ parameter to $\epsilon_0$.
\end{proof}

\begin{lemma}\label{lem:top-conv}
Let $(K,d)$ be a compact metric space and $h : K \to \mathbb{R}_{\geq 0}$ continuous. Define its zero set $Z = \{x \in K : h(x) = 0\}$. For all $\epsilon > 0$, there exists $\delta > 0$ such that $h(x) < \delta$ implies $d(x, Z) < \epsilon$. 
\end{lemma}

\begin{proof}[Proof by contradiction]
Suppose there exists some sequence $x_n$ that remains bounded away from $Z$, so that $d(x_n, Z) \geq \epsilon$, but $h(x_n)$ converges to zero. Then, by compactness, there is a convergent subsequence $x_{n_k} \to x$. By continuity, $h(x) = 0$, so that $x \in Z$. This is a contradiction; all the $x_n$ are $\epsilon$-bounded away from $Z$.
\end{proof}

\begin{lemma}[Second Borel-Cantelli lemma, \cite{durrett2019probability}] \label{lem:borel-cantelli}
Let $(\Omega, \mathcal{F}, P)$ be a probability space. If $(\mathcal{F}_n)_{n \geq 0}$ be a filtration with $\mathcal{F}_0 = \{\emptyset, \Omega\}$ and $(B_n)_{n\geq 1}$ a sequence of events with $B_n \in \mathcal{F}_n$, then:
\[\big\{B_n\ \textrm{occurs infinitely often}\big\} = \bigg\{\sum_{n=1}^\infty P(B_n\,|\, \mathcal{F}_{n-1}) = \infty\bigg\}.\]
\end{lemma}
\section{Analysis of the generalized online Lloyd's algorithm}  \label{sec:lloyd}
We now prove the convergence for the generalized online Lloyd's learning rate reproduced here:
\begin{equation}
    H_j^{(n+1)} = \frac{\ind{\{I^{(n+1)} = j\}}}{\max\{n\widehat{P}_j^{(n)}, t_n\}} \qquad \textrm{and}\qquad   \widehat{P}_j^{(n)} = \frac{1}{s_n} \sum_{n_\circ\leq n' < n} \ind\{I^{(n' + 1)} = j\}. \tag{\ref{eqn:learning-rate-empirical}}
\end{equation}
where we let $n_\circ = n - s_n$, and where $s_n$ and $t_n$ are non-decreasing sequences. Let $P_j^{(n)} := P_j(W^{(n)})$.

\lloydconvergence*

\begin{remark}[Existence of $s_n$ and $t_n$]
It is fairly easy to construct sequences $s_n$ and $t_n$ satisfying the condition of \Cref{thm:lloydconvergence}. In particular, let $s_n = n^\alpha$ and $t_n = n^\beta$, where $2/3 <\alpha < \beta < 1$.
\end{remark}

We show convergence by verifying conditions (A1) and (A2) of \Cref{thm:iterates-conv}. The bulk of our effort is spent on (A1). Here is a brief guide to the objects in this analysis. Recall the form of (A1):
\[\Pr\left(\sum_{j \in [k]}\sum_{m \leq n < T(m)}  H_j^{(n+1)}  < r \quad \mathrm{and}\quad \sum_{m \leq n < T(m)} H_i^{(n+1)} > s \, \middle|\, \textcolor{gray}{\mathcal{F}_m, \|\nabla_{w_i} f(W^{(m)})\| \in [\epsilon,\epsilon_0)} \right) > c.\]
Since $H_j^{(n+1)}$ depends on the estimator $\widehat{P}_j^{(n)}$, there are two time units of analysis: (i) many short intervals of length $s_n$ from $n_\circ$ to $n$ used to compute the estimators, and (ii) the much longer interval from $m$ to $T(m)$ over which we aim to bound the behavior of the accumulated learning rates. 

It turns out that our ability to control $\widehat{P}_j^{(n)}$ depends on how smooth the maps $P_j: \mathcal{D}_R \to [0,1]$ are on a neighborhood of the trajectory of the iterates during the short intervals. The main issue is that the $P_j$'s are not nice everywhere on $\mathcal{D}_R$. All is not lost though, for (A1) requires these bounds only when $\|\nabla_{w_i} f(W^{(m)})\| \leq \epsilon_0$. Hope remains if $\{\|\nabla_{w_i} f\| \leq \epsilon_0\}$ lies in some region $K$ of $\mathcal{D}_R$ on which the maps $P_j$ are well-behaved. Indeed, we shall be able to find such a $K$ onto which we can restrict our analysis. But we cannot simply condition on a future event that the trajectories remain in $K$, since many of the tools we use from martingale analysis break if we do so. To handle this, let us define the notion of a \emph{core set}.

\begin{definition}[$r$-core set]
Given $S\subset \mathcal{D}_R$, we say that $S_\circ$ is an \emph{$r$-core set} of $S$ if for all $m, n \in \mathbb{N}$,
\begin{equation}
W^{(m)} \in S_\circ \quad\textrm{and}\quad \sum_{j \in [k]} \sum_{m \leq n^{\prime} < n} H_j^{(n^{\prime}+1)} < r \qquad \implies \qquad \forall m \leq n^{\prime}\leq n, \quad W^{(n^{\prime})} \in S\quad \mathrm{a.s.}
\end{equation}
\end{definition}
In other words, if we are presently in an $r$-core set $W^{(m)} \in S_\circ$, then we are guaranteed to remain in $S$ so long as the accumulated learning rate does not exceed $r$. 

\begin{remark}  \label{rmk:core-set}
Recall from \Cref{lem:lr-iterates-bound} that the displacement in iterates is bounded by the accumulated learning rate by an additional factor of $2R$. It follows that $S_\circ$ is an $r$-core set of $S$ whenever (i) $S_\circ$ is contained in $S$, and (ii) $S_\circ$ is separated from the boundary $\partial S$ by a distance of $2R\cdot r$. Here, $\partial S := \mathrm{closure}(S) \setminus \mathrm{interior}(S)$.
\end{remark}

If we find an $r_\circ$-core set $K_\circ$ of $K$, we can ensure that iterates remain in $K$ from times $n_\circ$ through $n$ whenever $W^{(n_\circ)}$ begins in $K_\circ$ and the accumulated learning rates do not exceed $r_\circ$. It turns out that eventually the accumulated learning rate may always be upper bounded by $r_\circ$; \Cref{lem:lr-ub-sn} shows that the accumulated learning rate over this short interval $n_\circ$ to $n$ converges to zero.  

This allows us to analyze $\widehat{P}_j^{(n)}$. For example, if $W^{(n_\circ)} \in K_\circ$, \Cref{lem:est-concentration} applies Azuma-Hoeffding's to show that the estimator is consistent. In fact, it is concentrated with high probability:
\begin{equation} \label{eqn:est-concentration}
\Pr\left(\vphantom{\bigg|}\left| \widehat{P}_j^{(n)} - P_j^{(n)}\right| \leq  a_n \,\middle|\, \textcolor{gray}{\mathcal{F}_{n_\circ}, \ W^{(n_\circ)} \in K_\circ} \right) > 1 - \frac{1}{n}.
\end{equation}
where $a_n \to 0$ is a sequence depending on $s_n$ and $t_n$ that converges to zero.

So far, our discussion has focused on the analyses over the short intervals. But, we also have to bound the behavior of the learning rates over the long interval from $m$ to $T(m)$. Here, we run into the same issue: at time $m$, we cannot condition on the future event that the iterates remain in $K_\circ$, which we need to control the individual learning rates. We need to be able to be able to choose $K_\circ$ and $K$ so that $\{\|\nabla_{w_i} f\| \leq \epsilon_0\}$ is an $r_0$-core set of $K_\circ$. This is in fact possible; we obtain a sequence of core sets seen in \Cref{fig:regions}.

\begin{figure}[h]
    \centering
    \begin{tikzpicture}[y=0.80pt, x=0.80pt,yscale=-2.5, xscale=2.5, inner sep=0pt, outer sep=0pt, trim left = 2cm]

  \path[draw=black,fill=cf9f9f9,line cap=butt,line join=miter,line width=0.299pt] (54.9174,61.7760) .. controls (34.8042,74.5615) and (18.9199,110.7493) .. (28.0229,137.9947) .. controls (31.3683,148.0074) and (39.9622,156.3286) .. (49.8747,156.1061) .. controls (66.2536,155.7384) and (79.0581,136.9671) .. (93.5782,130.4483) .. controls (115.6379,120.5446) and (140.3069,127.6766) .. (162.4953,117.6194) .. controls (172.1698,113.2343) and (178.6216,107.5658) .. (180.6491,95.3575) .. controls (190.9266,33.4721) and (99.8417,33.6483) .. (54.9174,61.7760) -- cycle;

  \path[draw=black,fill=cb3b3b3,line cap=butt,line join=miter,line width=0.260pt] (83.0743,77.6287) .. controls (-43.9613,120.3447) and (122.5160,129.0018) .. (157.4525,98.7737) .. controls (163.0799,93.9047) and (164.9665,85.6292) .. (163.7028,78.5481) .. controls (156.7449,39.5602) and (119.4536,65.5231) .. (83.0743,77.6287) -- cycle;

  \path[draw=black,fill=c808080,line cap=butt,line join=miter,line width=0.242pt] (95.1801,79.2379) .. controls (14.8689,110.9273) and (118.2480,105.3799) .. (134.3191,84.5970) .. controls (153.0899,60.3230) and (111.7564,72.2878) .. (95.1801,79.2379) -- cycle;

  \path[draw=black,fill=c1a1a1a,line cap=butt,line join=miter,line width=0.212pt] (116.7903,79.1194) .. controls (92.9864,85.1016) and (116.2954,91.5117) .. (124.8235,82.4151) .. controls (130.9146,75.9179) and (122.3416,77.6839) .. (116.7903,79.1194) -- cycle;

  \path[draw=cffffff,line cap=butt,line join=miter,line width=0.652pt,miter limit=4.00] (78.3488,96.0089) .. controls (80.2461,95.7794) and (79.9236,96.1861) .. (80.5583,97.2910) .. controls (80.8274,97.7594) and (81.4587,97.2911) .. (81.7858,97.5047) .. controls (82.1999,97.7750) and (81.9663,98.4264) .. (82.2768,98.7868) .. controls (83.9529,100.7318) and (83.8112,102.7934) .. (83.2589,105.1972) .. controls (83.0546,106.0862) and (84.9225,104.0462) .. (85.4684,103.2741) .. controls (86.1848,102.2609) and (87.9548,98.3594) .. (89.8874,98.3594) .. controls (90.3784,98.3594) and (89.7912,99.2224) .. (89.8874,99.6415) .. controls (90.1818,100.9225) and (90.6977,99.9194) .. (91.1150,100.2825) .. controls (92.0860,101.1277) and (89.9415,103.3525) .. (90.6240,104.3425) .. controls (91.4499,105.5406) and (93.9231,101.2609) .. (95.0430,101.9920) .. controls (95.6977,102.4194) and (94.5104,103.5880) .. (95.0430,104.1288) .. controls (95.7199,104.8162) and (99.9408,106.0018) .. (100.6896,106.0520) .. controls (101.9822,106.1385) and (102.6266,105.2606) .. (103.8811,105.6246) .. controls (104.2693,105.7373) and (104.4557,106.2979) .. (104.8631,106.2656) .. controls (107.4619,106.0600) and (105.6196,104.7502) .. (108.3002,104.9836) .. controls (108.6457,105.0137) and (108.7017,105.5451) .. (109.0367,105.6246) .. controls (109.5874,105.7553) and (113.8430,105.7360) .. (114.6833,105.4109) .. controls (117.3340,104.3855) and (118.2357,103.0793) .. (121.3119,102.6330) .. controls (121.7155,102.5745) and (122.1595,102.5008) .. (122.5394,102.6330) .. controls (122.7673,102.7124) and (122.3012,103.3259) .. (122.5394,103.2741) .. controls (124.3596,102.8780) and (127.2657,102.3886) .. (128.9225,101.5646) .. controls (130.0197,101.0189) and (130.5784,99.1504) .. (131.6230,98.7868) .. controls (133.8714,98.0040) and (135.6175,98.4575) .. (136.7786,96.4363);

  \node [label=$\mathcal{D}_R$] at (183, 122) {};
  \node [label=$K$] at (163.5, 109) {}; 
  \node [label=$K_\circ$] at (147, 98) {}; 

\end{tikzpicture}
    \captionsetup{width=0.9\textwidth,singlelinecheck=off}
    \caption[.]{A sequence of subsets: $\{\nabla f = 0\} \ \textcolor{gray}{\subset}\  \{\|\nabla_{w_i} f\| \leq \epsilon_0\} \ \textcolor{gray}{\subset}\ K_\circ \ \textcolor{gray}{\subset}\  K \ \textcolor{gray}{\subset}\ \mathcal{D}_R$.
    The page is $\mathcal{D}_R$. As $P_j$ is not well-behaved over all of $\mathcal{D}_R$, we construct a compact subset $K$ (\textit{light gray}) over which the maps $P_j$ are $L$-Lipschitz. $K$ contains an $r_\circ$-core set $K_\circ$ (\textit{gray}), allowing \Cref{lem:est-concentration} to control the behavior of estimators over short intervals of length $s_n$ when $W^{(n_\circ)} \in K_\circ$. To control the learning rates over the long interval $m$ to $T(m)$, we chose $K_\circ$ and $K$ so that $\{\|\nabla_{w_i} f\|\leq \epsilon_0\}$ (\textit{dark gray}) is an $r_0$-core set of $K_\circ$. We show in \Cref{lem:lr-upper-bound} that when iterates start within this set, then they do not exit $K_\circ$ with constant probability during the long interval. The white squiggly line depicts the trajectory of such a sequence of iterates. Notice that $\{\|\nabla_{w_i}f\|\leq \epsilon_0\}$ contains the set of stationary points (\textit{black}).}
    \label{fig:regions}
\end{figure}
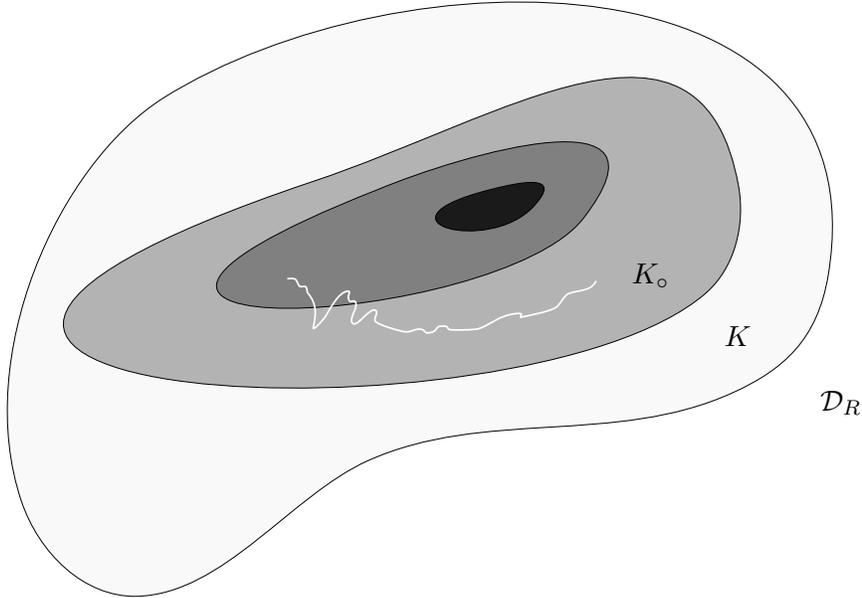

\subsection{Proof of \Cref{thm:lloydconvergence}}
Fix $\epsilon_0 > 0$ so that $K := \{\|\nabla_{w_i} f\| \leq 3\epsilon_0\}$ is compact; such an $\epsilon_0$ exists by \Cref{lem:compact-eps}. Because each of the $P_j : \mathcal{D}_R \to [0,1]$ is locally Lipschitz by \Cref{lem:local-lipschitz}, there exists a constant $L > 0$ so that they are all $L$-Lipschitz on $K$. Put $K_\circ := \{\|\nabla_{w_i} f\| \leq 2 \epsilon_0\}$. Then $K_\circ$ is bounded away from $\partial K := \{\|\nabla_{w_i} f\| = 3\epsilon_0\}$, since both are disjoint, non-empty, and compact sets, by \Cref{lem:non-empty-level-set}. Thus, \Cref{rmk:core-set} implies that $K_\circ$ is an $r_\circ$-core set of $K$ for some $r_\circ > 0$.
Similarly, $\{\|\nabla_{w_i} f\| \leq \epsilon_0\}$ is an $r_0$-core set of $K_\circ$ for some $r_0 > 0$.

We also define the sequence:
\begin{equation} \label{eqn:thresholds}
a_n := c \cdot \left(\frac{1}{t_{n_\circ}} + \frac{s_n \log s_n}{n}\right)\qquad\textrm{and}\qquad c:= \max\{1,256kRL\}.
\end{equation}
For any $r > 0$, define the function $T_r : \mathbb{N} \to \mathbb{N}$ so that $T_r(m)$ is the unique natural number so that:
\begin{equation} \label{eqn:time-bound}
    \sum_{m \leq n \textcolor{blue}{<} T_r(m)} \frac{1}{n} \leq r < \sum_{m \leq n \textcolor{orange}{\leq} T_r(m)} \frac{1}{n}.
\end{equation}

Because $\{\|\nabla_{w_i}f \|\leq \epsilon_0\}$ is an $r_0$-core set of $K_\circ$, the following lemma shows that we can choose $T$ so that eventually, whenever $W^{(m)} \in \{\|\nabla_{w_i} f\| \leq \epsilon_0\}$, then iterates will remain in $K_\circ$ for the whole duration from $m$ through $T(m)$ with constant probability. In fact, it more generally verifies the first half of condition (A1).

\begin{lemma} \label{lem:lr-upper-bound}
Let $H_j^{(n)}$ and $\widehat{P}_j^{(n)}$ be defined as in (\ref{eqn:learning-rate-empirical} and \ref{eqn:learning-rate-empirical-p}), and let $s_n$ and $t_n$ be non-decreasing sequences in $\mathbb{N}$. Let $\epsilon_0, K, K_\circ, L, r_\circ, r_0, a_n, c$ be given as above. Let $ 0 < r < \min \{\ln 2, r_0\}$. Assume that there exists $m_0 \in \mathbb{N}$ such that for all $n \geq m_0$, the following hold:
\begin{enumerate}
    \item[(a)] $4 n^{2/3} (\log 2n)^{1/3} \leq s_n \leq \frac{1}{2} n - 1$
    \item[(b)] $a_n < \min\{cr_\circ/16k, t_n/n\}$
    \item[(c)] $s_{2n} / t_n < r/12k$.
\end{enumerate}
Then, the function $T \equiv T_{Cr}$ where $C = 1/18 k$ satisfies for all $m \geq m_0$:
\[\Pr\left(\sum_{j \in [k]} \sum_{m \leq n < T(m)} H_j^{(n+1)} \geq r\,\middle|\, \mathcal{F}_m, \|\nabla_{w_i} f(W^{(m)})\| \leq \epsilon_0\right) \leq \frac{1}{3}.\]
\end{lemma}
\begin{proof}[Proof of \Cref{lem:lr-upper-bound}]
The essence of the proof will be to apply Markov's inequality by bounding the expectation of the summed learning rates. Since $H_j^{(n+1)}$ is of the form:
\[H_j^{(n+1)} = \frac{\ind\{I^{(n+1)} = j\}}{\max\{n\widehat{P}_j^{(n)}, t_n\}},\]
we upper bound it via the concentration result (\ref{eqn:est-concentration}) of \Cref{lem:est-concentration}, which lower bounds $\widehat{P}_j^{(n)}$ when iterates have not strayed out of $K_\circ$ by time $n_\circ$. Then, we apply Markov's to a related stopped process that sets learning rates to zero once iterates exit $K_\circ$. Let $(Z_n)_{n > m}$ be the accumulated learning rates from time $m$,
\begin{equation*}
Z_n = \sum_{j \in [k]}\sum_{m \leq n' < n} H_j^{(n'+1)}.
\end{equation*}
Define $\tau$ as the exit time from $K_\circ$ and $(Z_{n\wedge \tau})_{n > m}$ be the stopped process:
\begin{equation} \label{eqn:lr-total}
\tau := \min_{n \geq m} \, \{W^{(n)} \notin K_\circ\} \qquad \textrm{and}\qquad Z_{n\wedge \tau} = \sum_{j \in [k]}\sum_{m \leq n' < n \wedge \tau} H_j^{(n'+1)},
\end{equation}
where $n \wedge \tau := \min\{n, \tau\}$. We conditioned on $W^{(m)} \in \{\|\nabla_{w_i} f\| \leq\epsilon_0\}$ to be initially in an $r_0$-core set of $K_\circ$. So, the iterates remain in $K_\circ$ through iteration $n$ if $Z_n < r_0$. We claim that if $r < r_0$, then the events $\{Z_n < r\}$ and $\{Z_{n \wedge \tau} < r\}$ are equal. Indeed, we have that if $Z_n < r < r_0$, then $Z_{n\wedge \tau} = Z_n$. And because the accumulated learning rate when the process stops $Z_\tau$ must be at least $r_0$, we also have that if $Z_{n\wedge \tau} < r < r_0$, then the process has not stopped yet, so $Z_{n\wedge \tau} = Z_n$. Thus:
\[\Pr\left(Z_n \geq r\,\middle|\, \textcolor{gray}{\mathcal{F}_m, \|\nabla_{w_i} f(W^{(m)})\| \leq\epsilon_0}\right) = \Pr\left(\vphantom{\max_{m < n' \leq n}\,  Z_{n'} > r} Z_{n\wedge \tau} \geq r\,\middle|\, \textcolor{gray}{\mathcal{F}_m, \|\nabla_{w_i} f(W^{(m)})\| \leq\epsilon_0}\right).\]
We now bound the right-hand side by bounding the expected value of $Z_{T(m) \wedge \tau}$ and applying Markov's. The expected value of $Z_{T(m) \wedge \tau}$ can be bounded by considering each term within the summation (\ref{eqn:lr-total}) individually:
\begin{align*} 
&\E\left[Z_{T(m) \wedge \tau}\,\middle|\, \textcolor{gray}{\mathcal{F}_m, \|\nabla_{w_i} f(W^{(m)})\| \leq \epsilon_0}\right] 
\\&\qquad\qquad \leq \sum_{j \in[k]}\sum_{m \leq n < T(m)} \E\left[H_j^{(n+1)} \,\middle|\, \textcolor{gray}{\mathcal{F}_m, \|\nabla_{w_i} f(W^{(m)})\| \leq \epsilon_0}, \tau > n_\circ\right].
\end{align*}
Consider two intervals: (1) a warm-up interval $m \leq n \leq m+s_{T(m)}$ during which $n_\circ < m$ is possible, and (2) the tail interval $m + s_{T(m)} < n < T(m) \wedge \tau$ during which $n_\circ \geq m$ holds. For interval (1), we use the coarse bound ${H_j}^{\!(n+1)} \leq t_m^{-1}$. In interval (2), for any center $j \in [k]$ and iteration $m+s_{T(m)} \leq n < T(m)$, we have:
\begin{align*}
    &\E\left[H_j^{(n+1)} \,\middle|\, \textcolor{gray}{\mathcal{F}_m, \|\nabla_{w_i} f(W^{(m)})\|\leq \epsilon_0}, \tau > n_\circ\right] 
    \\&\qquad\qquad\overset{(i)}{=} \E\left[\E\left[H_j^{(n+1)} \,\middle|\, \mathcal{F}_n\right]\,\middle|\, \textcolor{gray}{\mathcal{F}_m, \|\nabla_{w_i} f(W^{(m)})\|\leq \epsilon_0}, \tau > n_\circ\right]
    \\&\qquad\qquad\overset{(ii)}{=} \E\left[\frac{P_j^{(n)}}{\max\{n\widehat{P}_j^{(n)}, t_n\}}\,\middle|\, \textcolor{gray}{\mathcal{F}_m, \|\nabla_{w_i} f(W^{(m)})\|\leq\epsilon_0}, \tau > n_\circ\right]
    \\&\qquad\qquad\overset{(iii)}{\leq} \frac{3}{n}.
\end{align*}
where (i) follows from the tower law for conditional expectations, (ii) from plugging in the form of $H_j^{(n+1)}$, and (iii) we must prove. But assuming this to be the case, we then have the upper bound:
\begin{align*}
    \E\left[Z_{T(m) \wedge \tau}\,\middle|\, \textcolor{gray}{\mathcal{F}_m, \|\nabla_{w_i}f(W^{(m)})\| \leq \epsilon_0}\right] &\leq \frac{k (s_{T(m)} + 1)}{t_m} + \sum_{m \leq n < T(m)} \frac{3k}{n} \quad \leq \quad  \frac{r}{3},
\end{align*}
where the last inequality holds because $m_0$ is sufficiently large so that $s_{2m}/t_m < r/12k$ in condition (c) holds, and because $T \equiv T_{Cr}$, where $C = 1/18k$. In particular, we assumed that $Cr < \ln 2$, so \Cref{cor:analytic-tr} shows that $T(m)\leq 2m$; as $s_n$ is non-decreasing, $s_{T(m)} + 1 \leq 2s_{2m}$. Thus, the first term is upper bounded by $r/6$. The second term is less than $r/6$ by the definition of $T$. The lemma follows from Markov's inequality. 

Only inequality (iii) above is left. Since $n_\circ \geq m$, by the tower law again, it suffices to show:
\[\E\left[\frac{P_j^{(n)}}{\max\{n\widehat{P}_j^{(n)}, t_n\}}\,\middle|\, \mathcal{F}_{n_\circ},\  \tau > n_\circ\right] \leq \frac{3}{n}.\]
Note that as $\tau > n_\circ$, we have $W^{(n_\circ)} \in K_\circ$. This along with conditions (a) and (b) shows that \Cref{lem:lr-ub-sn} and \Cref{lem:est-concentration} may be applied; we use them as follows. Consider two cases separately:
\[\{P_j^{(n_\circ)} \leq a_n \} \qquad \textrm{and}\qquad \{P_j^{(n_\circ)} > a_n\}.\] 
\Cref{lem:lr-ub-sn} shows that $P_j^{(n_\circ)}$ and $P_j^{(n)}$ are almost surely within $a_n/8$ of each other. So, in the first:
\begin{align*}
    \E\left[\frac{P_j^{(n)}}{\max\{n\widehat{P}_j^{(n)}, t_n\}}\,\middle|\, \textcolor{gray}{\mathcal{F}_{n_\circ},\ \tau > n_\circ}\ P_j^{(n_\circ)} \leq a_n\right]
    &\leq \frac{9}{8}\frac{a_n}{t_n} \leq \frac{3}{n},
\end{align*}
where the last inequality holds because we assumed that $ a_n/t_n \leq \frac{1}{n}$. 

In the second case, \Cref{lem:est-concentration} shows that $P_j^{(n)} / \widehat{P}_j^{(n)} < \textcolor{blue}{2}$ with probability at least $1 - \frac{1}{n}$. Because $H_j^{(n+1)} \leq 1$, the failure mode contributes at most $\textcolor{orange}{\frac{1}{n}}$ to the expectation:
\begin{align*}
    \phantom{+}\E\left[\frac{P_j^{(n)}}{\max\{n\widehat{P}_j^{(n)}, t_n\}}\,\middle|\, \textcolor{gray}{\mathcal{F}_{n_\circ}, \ \tau > n_\circ,}\ P_j^{(n_\circ)} > a_n\right]
    &\leq \frac{\textcolor{blue}{2}}{n} + \textcolor{orange}{\frac{1}{n}} \leq \frac{3}{n}.\qedhere
\end{align*}
\end{proof}

\begin{remark} \label{rmk:core-core-set}
As we mentioned earlier, when $r$, $T$ and $m_0$ are defined as in \Cref{lem:lr-upper-bound}, this result shows that the iterates from $m$ through $T(m)$ remain in $K_\circ$ with probability at least $2/3$ when $m \geq m_0$. This is because we conditioned on $W^{(m)} \in \{\|\nabla_{w_i} f\| \leq \epsilon_0\}$, which is an $r_0$-core set of $K_\circ$, and we assumed $r < r_0$.
\end{remark}

Introducing a few more conditions, which we highlight in blue, leads to all of (A1).

\begin{lemma} \label{lem:lr-lower-bound}
Let $H_j^{(n)}$ and $\widehat{P}_j^{(n)}$ be defined as in (\ref{eqn:learning-rate-empirical} and \ref{eqn:learning-rate-empirical-p}), and let $s_n$ and $t_n$ be non-decreasing sequences in $\mathbb{N}$. Let $\epsilon_0, K, K_\circ, L, r_\circ, r_0, a_n, c$ be given as above. Let $\textcolor{blue}{\epsilon \in (0, \epsilon_0)}$ and also let $ 0 < r < \min \{\ln 2, r_0,  \textcolor{blue}{\epsilon / 8R^2L, \frac{1}{C}\ln \frac{7}{6}}\}$. Set $\textcolor{blue}{s = Cr/8}$. Assume that there exists $m_0 \in \mathbb{N}$ such that for all $n \geq m_0$, the following hold:
\begin{enumerate}
    \item[(a)] $4 n^{2/3} (\log 2n)^{1/3} \leq s_n \leq \min\{\frac{1}{2} n - 1, \textcolor{blue}{ \frac{1}{2}(e^{Cr/2} - 1) n - e^{Cr/2}}\}$
    \item[(b)] $a_n < \min\{cr_\circ/16k, t_n/n,\textcolor{blue}{ \epsilon/4R}\}$
    \item[(c)] $s_{2n} / t_n < r/12k$.
    \item[(d)] $\textcolor{blue}{e^{Cr/2} \sqrt{2n \ln 6} / s < t_n < n \epsilon /8R}$.
\end{enumerate}
Then, the function $T \equiv T_{Cr}$ where $C = 1/18 k$ satisfies for all $m \geq m_0$:
\[\Pr\left(\sum_{j \in [k]} \sum_{m \leq n < T(m)} H_j^{(n+1)} < r \quad \mathrm{and}\quad \sum_{m \leq n < T(m)} H_i^{(n+1)} > s \,\middle|\, \mathcal{F}_m, \|\nabla_{w_i} f(W^{(m)})\| \in [\epsilon,\epsilon_0)\right) > \frac{1}{3}.\]
\end{lemma}

\begin{proof}[Proof of \Cref{lem:lr-lower-bound}]
The learning rate $H_i^{(n+1)}$ has conditional expectation:
\[\E\left[H_i^{(n+1)} \,\middle|\, \textcolor{gray}{\mathcal{F}_n}\right] = \E\left[\frac{\ind\{I^{(n+1)} = i\}}{\max\{ n \widehat{P}_i^{(n)}, t_n\}}\,\middle|\, \textcolor{gray}{\mathcal{F}_n}\right] = \frac{P_i^{(n)}}{\max\{n\widehat{P}_i^{(n)}, t_n\}}.\]
Thus, the sequence $H_i^{(n+1)} - \E\left[H_i^{(n+1)}\,\middle|\, \textcolor{gray}{\mathcal{F}_n}\right]$ is a martingale difference sequence with bounded increments:
\[ \left|\frac{\ind\{I^{(n+1)} = i\} - P_i^{(n)}}{\max\{n\widehat{P}_i^{(n)}, t_n\}}\right| \leq  \frac{1}{t_n}.\]
Define $\mu$ and $\nu$ as follows:
\[\mu := \sum_{m \leq n < T(m)} \frac{P_i^{(n)}}{\max\{n\widehat{P}_i^{(n)}, t_n\}} \qquad \textrm{and}\qquad \nu:= \left( \sum_{m \leq n < T(m)} \frac{1}{t_n^2}\right)^{-1}.\]
Azuma-Hoeffding's implies that the accumulated learning rates for the $i$th center concentrates about $\mu$,
\begin{align} \label{eqn:ah-lr-lb}
    \Pr\left(\sum_{m \leq n < T(m)} H_i^{(n+1)} > \mu - s \,\middle|\, \textcolor{gray}{\mathcal{F}_m, \|\nabla_{w_i} f(W^{(m)})\| \in [\epsilon,\epsilon_0)}\right) 
    &> 1 - \exp\left(- \frac{1}{2}\nu s^2\right) > \frac{5}{6},
\end{align}
where the last inequality follows from:
\[\nu = \left( \sum_{m \leq n < T(m)} \frac{1}{t_n^2}\right)^{-1} \overset{(i)}{\geq} \frac{1}{e^{Cr} - 1}\frac{t_m^2}{m} \overset{(ii)}{>} \frac{2 \ln 6}{s^2}, \]
since (i) $T(m) - m \leq (e^{Cr} - 1)m$ by \Cref{cor:analytic-tr} and $t_n$ is non-decreasing, and (ii) $t_n > e^{Cr/2} \sqrt{2n \ln 6} / s$ by assumption (d). We also claim that:
\begin{equation} \label{eqn:lri-mean}
\Pr\left(\sum_{j \in [k]} \sum_{m \leq n < T(m)} H_j^{(n+1)} < r  \quad \textrm{and}\quad \mu > 2s \,\bigg|\, \textcolor{gray}{\mathcal{F}_m, \|\nabla_{w_i} f(W^{(m)})\| \in [\epsilon,\epsilon_0)}\right) > \frac{1}{2}.
\end{equation}
Assuming the claim, we obtain the desired result by combining (\ref{eqn:ah-lr-lb}) and (\ref{eqn:lri-mean}) by a union bound.

Now, only the claim remains, but let's first reduce notation. Denote the two events in (\ref{eqn:lri-mean}) by $\Xi$ and $\mathrm{E}$,
\[\Xi = \displaystyle \left\{\sum_{j \in [k]} \sum_{m \leq n < T(m)} H_j^{(n+1)} < r\right\} \qquad \textrm{and}\qquad \mathrm{E} := \left\{\sum_{m \leq n < T(m)} \frac{P_i^{(n)}}{\max\{n\widehat{P}_i^{(n)}, t_n\}} > 2s\right\},\]
and let $\mathrm{F} = \textcolor{gray}{\mathcal{F}_m, \|\nabla_{w_i} f(W^{(m)})\| \in [\epsilon,\epsilon_0)}$ denote the conditioning; (\ref{eqn:lri-mean}) states that $\Pr\big(\Xi \cap \mathrm{E}\,|\, \textcolor{gray}{\mathrm{F}}\big) > 1/2$. From \Cref{lem:lr-upper-bound}, we have $\Pr(\Xi\,|\,\textcolor{gray}{\mathrm{F}}) > 2/3$. Thus, we just need to show that the event $\mathrm{E}|\textcolor{gray}{\mathrm{F}}$ also likely holds. This turns out to be the case if for nearly all iterations between $m$ and $T(m)$, the conditional expectation satisfies:
\begin{equation} \label{eqn:lr-i-lb}
\frac{P_i^{(n)}}{\max\{n \widehat{P}_i^{(n)}, t_n\}} > \frac{1}{2n}.
\end{equation}
Again to reduce notation, denote the event in (\ref{eqn:lr-i-lb}) by $\mathrm{E}_n$. Then, $\mathrm{E}$ occurs if all events $\mathrm{E}_n$ occur over times $m + s_{T(m)} \leq n < T(m)$. This is due the definition of $T$, which implies the following:
\begin{align*}
    \sum_{m + s_{T(m)} \leq n < T(m)} \frac{1}{2n} &\overset{(i)}{\geq} \frac{1}{2} \log \frac{T(m)}{m + s_{T(m)}}
    \\&\overset{(ii)}{\geq} \frac{1}{2} \log \frac{e^{Cr}( m-1)}{m + s_{T(m)}}
    \\&\overset{(iii)}{\geq} 2s,
\end{align*}
where (i) follows from \Cref{lem:harmonic-partial-sum}, (ii) from \Cref{cor:analytic-tr}, and (iii) from setting $s = Cr/8$ and our choice of upper bound on $s_n$. In particular, because $r < \ln 2$, \Cref{cor:analytic-tr} shows that $T(m) \leq 2m$. Since $s_n$ is non-decreasing, $s_{T(m)} \leq s_{2m} \leq (e^{Cr/2} - 1)m - e^{Cr/2}$. A little bit of algebra verifies (iii).


The natural way to prove (\ref{eqn:lri-mean}) would then be to union bound the probability of failure:
\begin{align*}
    \Pr\left(\Xi^c \cup \mathrm{E}^c \,\bigg|\, \textcolor{gray}{\mathrm{F}}\right) 
    &\leq \Pr\left(\Xi^c\,\bigg|\, \textcolor{gray}{\mathrm{F}}\right) + \sum_{m+ s_{T(m)} \leq n < T(m)}\Pr\left(\mathrm{E}_n^c\,\bigg|\, \textcolor{gray}{\mathrm{F}}\right).
\end{align*}
This turns out to be too coarse of a bound for us; there is too much overcounting of certain outcomes in $\Xi^c$ that are also contained in $\mathrm{E}_n^c$. Instead, we use a finer union bound---since the first term $\Pr(\Xi^c \,|\, \textcolor{gray}{\mathrm{F}})$ already accounts for the bad outcomes in $\Xi^c$, the remaining sum needs only measure the outcomes in $\mathrm{E}_n^c \cap \Xi$. In fact, we only loosen the bound if we measure outcomes in $\mathrm{E}_n^c \cap \Xi_n$, for a superset $\Xi_n \supset \Xi$. Thus:
\begin{align}
    \Pr\left(\Xi^c \cup \mathrm{E}^c \,\bigg|\, \textcolor{gray}{\mathrm{F}}\right) 
    &\leq \Pr\left(\Xi^c\,\bigg|\, \textcolor{gray}{\mathrm{F}}\right) + \sum_{m+ s_{T(m)} \leq n < T(m)}\Pr\left(\mathrm{E}_n^c \cap \Xi\,\bigg|\, \textcolor{gray}{\mathrm{F}} \right) \notag
    \\&\leq \Pr\left(\Xi^c\,\bigg|\, \textcolor{gray}{\mathrm{F}}\right) + \sum_{m+ s_{T(m)} \leq n < T(m)}\Pr\left(\mathrm{E}_n^c \cap \Xi_n\,\bigg|\, \textcolor{gray}{\mathrm{F}} \right) \notag
    \\& \leq \Pr\left(\Xi^c\,\bigg|\, \textcolor{gray}{\mathrm{F}}\right) + \sum_{m+ s_{T(m)} \leq n < T(m)}\Pr\left(\mathrm{E}_n^c \,\bigg|\, \textcolor{gray}{\mathrm{F}}, \Xi_n\right), \label{eqn:union-bound}
\end{align}
where in the last step, we use the general fact that $\Pr(A \cap B) \leq \Pr(A \,|\, B)$.

\Cref{lem:lr-upper-bound} bounds the first term in (\ref{eqn:union-bound}) with $\Pr(\Xi^c\,|\, \textcolor{gray}{\mathrm{F}}) \leq 1/3$. For the others, notice that conditioned on $\mathrm{F}$, the event $\Xi$, which bounds the accumulated learning rates by $r$, implies the $\mathcal{F}_{n_\circ}$-measurable events:
\[\Xi_n := \left\{W^{(n_\circ)} \in K_\circ\quad \textrm{and}\quad P_i^{(n_\circ)} \geq \frac{\epsilon}{4R}\right\}.\]
This is because $r < \min\{r_0, \epsilon / 8R^2L\}$. If $\Xi\,|\,\textcolor{gray}{\mathrm{F}}$ occurs, the bound $r < r_0$ implies that all iterates remain in $K_\circ$, as discussed in \Cref{rmk:core-core-set}. Furthermore, since $P_i$ is $L$-Lipschitz on $K_\circ$, we also have for all $m \leq n \leq T(m)$,  
\[\left|P_i^{(n)} - P_i^{(m)}\right| \leq 2RL \cdot \sum_{j \in [k]} \sum_{m \leq n' < n} H_j^{(n'+1)} \leq 2RL r.\] 
Thus, $r < \epsilon / 8R^2L$ implies $2RLr < \epsilon / 4R$. So, $P_i^{(n)}$ is lower bounded because $P_i^{(m)} \geq \epsilon / 2R$. This comes from the gradient, $\nabla_{w_i} f(w) = P_i(w) \cdot \big(w_i - M_i(w)\big)$, and that the initial iterate satisfies $\|\nabla_{w_i}f(W^{(m)})\| \geq \epsilon$.

We claim that $\Pr(\mathrm{E}_n^c\,|\, \textcolor{gray}{\mathrm{F}}, \Xi_n) \leq \frac{1}{n}$. If this is the case, then (\ref{eqn:union-bound}) implies (\ref{eqn:lri-mean}):
\begin{align*}
    \Pr\left(\Xi^c \cup \mathrm{E}^c \,\bigg|\, \textcolor{gray}{\mathrm{F}}\right) &\leq \frac{1}{3} + \sum_{m + s_{T(m)} \leq n < T(m)} \frac{1}{n}
    \\&\overset{(i)}{\leq} \frac{1}{3} + \frac{T(m) - m}{m} \overset{(ii)}{\leq} \frac{1}{3} + \left(e^{Cr} - 1\right) \overset{(iii)}{\leq} \frac{1}{2}, 
\end{align*}
where we use (i) $\frac{1}{m} \geq \frac{1}{n}$ on this interval, (ii) $T(m) - m \leq (e^{Cr} - 1) m$, and (iii) $r < \frac{1}{C}\ln \frac{7}{6}$.

Now, all that is left is to verify that $\Pr(\mathrm{E}_n^c\,|\, \textcolor{gray}{\mathrm{F}}, \Xi_n)$ is bounded above:
\[\Pr\left(\frac{P_i^{(n)}}{\max\{n\widehat{P}_i^{(n)}, t_n\}} \leq \frac{1}{2n}\,\middle|\,\textcolor{gray}{\mathcal{F}_m, \|\nabla_{w_i} f(W^{(m)})\| \in [\epsilon,\epsilon_0)}, \ W^{(n_\circ)} \in K_\circ, \ P_i^{(n_\circ)} \geq \frac{\epsilon}{4R}\right) \leq \frac{1}{n}.\]
This is true by \Cref{lem:est-concentration}, which shows multiplicative concentration 
$ \frac{1}{2} P_i^{(n)} < \widehat{P}_i^{(n)} < 2 P_i^{(n)}$
with probability at least $1 - \frac{1}{n}$. This lemma applies because $\Xi_n$ is $\mathcal{F}_{n_\circ}$-measurable with $n_\circ > m$, since we only consider iterations $n$ between $m + s_{T(m)}$ and $T(m)$, and because $\epsilon / 4R > a_n$. Concentration implies $\mathrm{E}_n$: the left tail bound shows that $\max\{n \widehat{P}_i^{(n)}, t_n\} = n \widehat{P}_i^{(n)}$ as $t_n < n \epsilon /8R$ and $P_i^{(n)} \geq \epsilon/4R$; the right tail shows $P_i^{(n)} / n\widehat{P}_i^{(n)} < 1/2n$.
\end{proof}

We now verify condition (A2) of \Cref{thm:iterates-conv}. The proof remixes techniques used for (A1). If $\|\nabla_{w_i} f\|$ is lower bounded by a constant, then so is $P_i$. And so $\widehat{P}_i^{(n)}$ will also be lower bounded by a constant with high probability. On average $H_i^{(n+1)}$ is on the order of $n^{-1}$, whose sum diverges.

\begin{lemma} \label{lem:lr-inf-lb}
Let $H_j^{(n)}$ and $\widehat{P}_j^{(n)}$ be defined as in (\ref{eqn:learning-rate-empirical} and \ref{eqn:learning-rate-empirical-p}) and $\displaystyle \lim_{n \to \infty}\, \frac{\log n}{s_n} = \lim_{n\to\infty} \, \frac{t_n}{n} = 0$. Let $\epsilon_0 > 0$. Then:
\[ \liminf_{n \to \infty} \, \|\nabla_{w_i} f(W^{(n)})\| \geq \epsilon_0 \qquad \implies \qquad \sum_{n \in \mathbb{N}} H_i^{(n)} = \infty \quad \mathrm{a.s.}\]
\end{lemma}

\begin{proof}[Proof of \Cref{lem:lr-inf-lb}]
We saw in the proof of \Cref{lem:lr-lower-bound} that $P_i^{(n)} > \epsilon_0/2R$ is implied by: \[\|\nabla_{w_i} f(W^{(n)})\| > \epsilon,\]
since the gradient is $\nabla_{w_i} f(w) = P_i(w) \cdot \big(w_i - M_i(w)\big)$. Therefore, if the limit infimum condition holds, then there exists a random variable $N \in \mathbb{N}$ such that if $n > N$, then:
\[P_i^{(n)} \geq \frac{\epsilon_0}{4R}.\]
That is, the probability of updating the center $i$ eventually remains at least $\epsilon_0/4R$. Thus, $\widehat{P}_i^{(n)} > \epsilon_0/8R$ holds with high probability at any sufficiently large iteration. In particular, if $n$ is large enough to satisfy $n_\circ > N$ and $s_n > \frac{128 R^2}{\epsilon_0^2} \ln n$, then Azuma-Hoeffding's implies (\Cref{lem:est-concentration} uses the same technique),
\begin{align*}
    \Pr\left(\widehat{P}_j^{(n)} \leq \frac{\epsilon_0}{8R}\vphantom{\bigg|}\,\middle|\, \textcolor{gray}{\mathcal{F}_{n_\circ}}\right) \leq \Pr\left(\widehat{P}_j^{(n)} \leq \frac{1}{s_n} \sum_{n_\circ\leq n < n} P_j^{(n')} - \frac{\epsilon_0}{8R}\,\middle|\, \textcolor{gray}{\mathcal{F}_{n_\circ}}\right)
    \leq \exp\left( -\frac{s_n \epsilon_0^2}{128R^2} \right) \leq \frac{1}{n}.
\end{align*}
This bound on $\widehat{P}_i^{(n)}$ implies one on $H_i^{(n+1)}$. In particular, if $n$ is sufficiently large so that $t_n < n \epsilon_0 / 8R$, then:
\[\Pr\left(\textcolor{gray}{H_i^{(n+1)} = }\ \frac{P_i^{(n)}}{\max\{n\widehat{P}_i^{(n)}, t_n\}} = \frac{P_i^{(n)}}{n\widehat{P}_i^{(n)}}\ \textcolor{gray}{\geq \frac{1}{n} \frac{\epsilon_0}{4R}}\,\middle|\,\mathcal{F}_{n_\circ}\right) > 1 - \frac{1}{n},\]
where the inequality in gray comes from the lower bound $P_i^{(n)} \geq \epsilon_0/4R$ and the upper bound $\widehat{P}_i^{(n)} \leq 1$. And so, we have that for $n$ sufficiently large:
\[\Pr\left(\sum_{m \leq n' < T_r(m)} H_i^{(n'+1)} > c \,\middle|\, \mathcal{F}_{n_\circ} \right) > 1 - r,\]
where we may take $r < 1$ and we set $c = r\epsilon_0/ 4R$. This inequality follows directly from a union bound and the definition of $T_r$. Borel-Cantelli implies that the accumulated learning on the $i$th center increases by $c$ infinitely often, which implies that $\sum_{n \in \mathbb{N}} {H_i}^{(n)}$ diverges.
\end{proof}

\noindent Under assumptions on $s_n$ and $t_n$, the conditions of \Cref{lem:lr-lower-bound} and \Cref{lem:lr-inf-lb} are verified. Thus, conditions (A1) and (A2) of \Cref{thm:iterates-conv} are satisfied, proving \Cref{thm:lloydconvergence}.  \hfill $\blacksquare$

\vspace{12pt}
\begin{remark}[Generalized union bound]
The modified union bound used in \Cref{lem:lr-lower-bound} may be of generic interest: let $(\Omega, \mathcal{F}, P)$ be a probability space. Let $A, B, C \in \mathcal{F}$ be events such that $A^c \subset C$. Then:
\[P(A \cup B) \quad\textcolor{gray}{\overset{(i)}{=}}\quad P(A) + P(B \cap A^c) \quad\textcolor{gray}{\overset{(ii)}{\leq}}\quad P(A) + P(B \cap C) \quad\textcolor{gray}{\overset{(iii)}{\leq}}\quad P(A) + P(B\,|\,C),\]
where \textcolor{gray}{(i)} $A \cup B$ is the disjoint union $A \sqcup (B \cap A^c)$, \textcolor{gray}{(ii)} $B\cap A^c \subset B \cap C$, and \textcolor{gray}{(iii)} $P(B \cap C) \leq P(C) P(B\,|\,C)$.  This is useful because $P(B\,|\,C)$ may in general be easier to bound than $P(B\,|\, A^c)$, as was our case.
\end{remark}

\subsection{Consistency and concentration of $\widehat{P}_j^{(n)}$}
\label{sec:p-estimator}
The estimator $\widehat{P}_j^{(n)}$ for $P_j^{(n)} := P_j(W^{(n)})$ is consistent, provided $P_j$ is locally Lipschitz and that:
\begin{equation*} 
\frac{1}{t_n} \to 0 \qquad\textrm{and}\qquad \frac{s_n \log s_n}{n} \to 0, \qquad \textrm{as}\quad n\to\infty.
\end{equation*}
Specifically, we give non-asymptotic rates of concentration in \Cref{lem:est-concentration}. 

The estimator $\widehat{P}_j^{(n)}$ depends on the trajectory of the past $s_n$ iterates up to that point. In particular, since $I^{(n' + 1)}$ is drawn from $P(W^{(n')})$, Azuma-Hoeffding's shows that the estimator tends to concentrate around:
\begin{equation*}
\frac{1}{s_n} \sum_{n_\circ \leq n' < n} P_j^{(n')}.
\end{equation*}
Therefore, $\widehat{P}_j^{(n)}$ concentrates around $P_j^{(n)}$, as long as $P_j^{(n')}$ does not vary too much over $n_\circ \leq n' \leq n$. The amount of variation can be bounded because the maps $P_j : \mathcal{D}_R \to [0,1]$ are locally Lipschitz (\Cref{lem:local-lipschitz}). We just need to ensure that the iterates do not move too much---we achieve this by upper bounding the accumulated learning rates between $n_\circ$ and $n$, achieved by the next lemma: by bounding the learning rates, we can control the change in $P_j$ whenever iterates stay within a region $K$ on which $P_j$ is $L$-Lipschitz.

\begin{lemma} \label{lem:lr-ub-sn}
Let $H_j^{(n)}$ and $\widehat{P}_j^{(n)}$ be defined as in (\ref{eqn:learning-rate-empirical}). If $e \leq s_n + 1\leq \frac{n}{2}$, then:
\[\sum_{j \in [k]}\sum_{n_\circ \leq n' < n} H_j^{(n'+1)} \leq \frac{16k}{t_{n_\circ}} + \frac{16 k s_n \log s_n}{n} \quad \mathrm{a.s.}\]
Let $K \subset \mathcal{D}_R$ be given so that the restriction $P_j\big|_K : K \to [0,1]$ is $L$-Lipschitz. Conditioned on $W^{(n_\circ)}, \ldots, W^{(n)}$ remaining in $K$, then for all $n_\circ \leq n' \leq n$:
\begin{equation}\label{eqn:p-shift}
    \left|P_j^{(n')} - P_j^{(n)}\right| \leq 32kR L\cdot \left(\frac{1}{t_{n_\circ}} + \frac{s_n \log s_n}{n}\right)\quad\mathrm{a.s.}
\end{equation}
\end{lemma}

\begin{proof}
Fix $j \in [k]$. The following chain of inequalities holds almost surely:
\begin{align*}
    \sum_{n_\circ \leq n' < n} H_j^{(n'+1)} &\leq \sum_{n_\circ \leq n' < n} \frac{\ind\{I^{(n'+1)} = j\}}{\max\{\textcolor{blue}{n'} \cdot \textcolor{orange}{\widehat{P}_j^{(n')}}, t_{n'}\}}
    \\&\leq \frac{1}{t_{n_\circ}} + \sum_{n' = 1}^{s_n-1} \frac{1}{\textcolor{blue}{(n - s_n)} \cdot \textcolor{orange}{\frac{n'}{s_{n}}}}
    \leq  \frac{1}{t_{n_\circ}} + \frac{16 s_n \log s_n}{n}\quad \textrm{a.s.}
\end{align*}
The first equality expands the definition of the learning rate (\ref{eqn:learning-rate-empirical}). The next inequality comes the worst-case scenario where the $j$th center has had no recent updates (so that $\widehat{P}_j^{(n - s_n)} = 0$), and it is updated every single time following during this window of length $s_n$. We've re-indexed the sum by subtracting $n - s_n$ from the original index. For the first term, since $\widehat{P}_j{}^{(n - s_n)} = 0$, we use the bound $H_j^{(n_\circ + 1)} \leq t_{n_\circ}^{-1}$. And for the rest, we bound each $\widehat{P}_j^{(n')}$ by $\frac{1}{s_n}, \frac{2}{s_n},\ldots, \frac{s_n - 1}{s_n}$, respectively. This is the worst-case scenario, since delaying an update simply introduces a zero in the sum and shifts the rest of the bounds to the right. The final inequality upper bounds the partial sums of the harmonic series, and uses the assumption $e \leq s_n \leq \frac{n}{2}$.

To obtain (\ref{eqn:p-shift}), \Cref{lem:lr-iterates-bound} converts a learning rate bound to one on $\|W^{(n')} - W^{(n)}\|$, introducing a factor of $2R$. Then, Lipschitz continuity bounds $|P_j^{(n')} - P_j^{(n)}|$, introducing a factor of $L$.
\end{proof}

The analysis to show that $\widehat{P}_j^{(n)}$ concentrates around $P_j^{(n)}$ would be quite straightforward if $P_j$ were globally Lipschitz---then we could use \Cref{lem:lr-ub-sn} to design conditions on $s_n$ and $t_n$ to force the accumulated learning rates to go to zero over periods of $s_n$,
\[\lim_{n \to \infty}\, \sum_{j \in[k]} \sum_{n_\circ \leq n' < n} H_j^{(n'+1)} = 0.\]
Thus over a small interval $s_n$, the iterates would remain close together, and the bias of $\widehat{P}_j^{(n)}$ would be forced to zero in the limit. And if $s_n \uparrow \infty$, then Azuma-Hoeffding's would imply increasingly tight concentration. 

Unfortunately, $P_j$ is not generally globally Lipschitz; the local Lipschitz constant at $w \in \mathcal{D}_R$ depends on the distances between centers $\|w_j - w_{j'}\|$, so we need to perform our analysis on a subset $K\subset \mathcal{D}_R$ on which $P_j$ is $L$-Lipschitz. While we need to know that the iterates $W^{(n_\circ)},\ldots, W^{(n)}$ remain in $K$, this event is not generally contained in $\mathcal{F}_{n_\circ}$. Directly conditioning on it would introduce new dependencies that prevent us from applying Azuma-Hoeffding's. We can overcome this issue by conditioning on an $\mathcal{F}_{n_\circ}$-measurable event contained within this event instead: that $W^{(n_\circ)}$ is contained in $K_\circ$, some $r_\circ$-core set of $K$.

\begin{lemma}[Estimator concentration]
\label{lem:est-concentration}
Let $H_j^{(n)}$ and $\widehat{P}_j^{(n)}$ be defined as in (\ref{eqn:learning-rate-empirical} and \ref{eqn:learning-rate-empirical-p}). Let $K \subset \mathcal{D}_R$ be given so that the restriction $P_j\big|_K : K \to [0,1]$ is $L$-Lipschitz. Let $K_\circ$ be an $r_\circ$-core set of $K$.
Let $c = \max\{1, 256 kRL\}$ and $a_n = c\cdot \left(\frac{1}{t_{n_\circ}} + \frac{s_n\log s_n}{n}\right)$. If  $s_n$ satisfies $ 4 n^{2/3} (\log 2n)^{1/3} \leq s_n \leq \frac{n}{2} - 1$ and $a_n < cr_\circ / 16k$, then:
\[\Pr\left(\left| \widehat{P}_j^{(n)} - P_j^{(n)}\right| <  \frac{3}{8} a_n\,\middle|\, \mathcal{F}_{n_\circ}, \ W^{(n_\circ)} \in K_\circ \right) > 1 - \frac{1}{n}.\]
In particular, a multiplicative bound holds:
\[\Pr\left(\frac{1}{2}  P_j^{(n)} < \widehat{P}_j^{(n)} < 2 P_j^{(n)}\,\middle|\, \mathcal{F}_{n_\circ}, \  P_j^{(n_\circ)} > a_n, \ W^{(n_\circ)} \in K_\circ \right) > 1 - \frac{1}{n}.\]
\end{lemma}
\begin{proof}
The following sequence during the interval $n_\circ \leq n' < n$ is a martingale difference sequence:
\[\ind\{I^{(n'+1)} = j\} - P_j^{(n')}.\]
In fact, since the event $\{W^{(n_\circ)} \in K_\circ\}$ is $\mathcal{F}_{n_\circ}$-measurable, we can condition on it, and the sequence remains a martingale difference sequence. Then, $\widehat{P}_j^{(n)}$ is concentrated:
\begin{align*}
\Pr\left(\bigg|\widehat{P}_j^{(n)} - \frac{1}{s_n} \sum_{n_\circ\leq n' < n} P_j^{(n')}\bigg| \geq \frac{a_n}{4}\,\middle|\, \textcolor{gray}{\mathcal{F}_{n_\circ},\ W^{(n_\circ)} \in K_\circ} \right) 
&\overset{(i)}{\leq} 2 \exp\left( -\frac{s_n a_n^2}{32}\right) 
\\&\overset{(ii)}{\leq} 2 \exp\left(- \frac{s_n^3}{64n^2}\right) 
\overset{(iii)}{\leq} \frac{1}{n},
\end{align*}
where (i) follows from Azuma-Hoeffding's, (ii) from $\frac{1}{32} a_n^2 \geq \frac{1}{64} s_n^2/n^2$ since $c \geq 1$, and (iii) from plugging in the lower bound on $s_n$ in the theorem statement. 

To complete the theorem, we need to relate $P_j^{(n')}$ to $P_j^{(n)}$, which we can do whenever the iterates remain in $K$. Indeed, we conditioned on $W^{(n_\circ)} \in K_\circ$, and we also have: 
\[\sum_{j \in [k]} \sum_{n_\circ \leq n' < n} H_j^{(n'+1)} \overset{(i)}{\leq} \frac{16k a_n}{c} \overset{(ii)}{<} r_\circ \quad \mathrm{a.s.},\]
where (i) is the first result of \Cref{lem:lr-ub-sn}, which we may apply since $4 \leq s_n \leq \frac{n}{2} - 1$, and (ii) we assumed that $a_n < cr_\circ/16k$. By the core-set property of $K_\circ$, the iterates remain in $K$ during this interval on which $P_j$ is $L$-Lipschitz. We can now apply the second result (\ref{eqn:p-shift}) of \Cref{lem:lr-ub-sn}, which shows that for all $n_\circ \leq n' \leq n$,
\begin{equation*} 
\left|P_j^{(n')} - P_j^{(n)}\right| \leq 32kRL \cdot \left(\frac{1}{t_{n_\circ}} + \frac{s_n \log s_n}{n}\right) \leq \frac{a_n}{8}\quad \mathrm{a.s.}
\end{equation*}
By triangle inequality:
\[\bigg| \frac{1}{s_n} \sum_{n_\circ \leq n' < n} P_j^{(n')} - P_j^{(n)}\bigg| \leq \frac{a_n}{8} \quad \mathrm{a.s.}\]
A further application of triangle inequality yields the desired additive concentration bound:
\[\Pr\left(\left|\widehat{P}_j^{(n)} - P_j^{(n)}\right| \geq \frac{3}{8} a_n\,\middle|\, \textcolor{gray}{\mathcal{F}_{n_\circ},  \ W^{(n_\circ)} \in K_\circ} \right) \leq \frac{1}{n}.\]
If we further condition on the $\mathcal{F}_{n_\circ}$-measurable event $\{P_j^{(n_\circ)} > a_n\}$, then (\ref{eqn:p-shift}) implies $P_j^{(n)} > \frac{7}{8}a_n$, from which we also obtain the multiplicative bound.
\end{proof}

\subsection{Local Lipschitzness of $P_j$}

\begin{lemma}[$P_j$ is locally Lipschitz] \label{lem:local-lipschitz}
Let $p$ be a density supported in the closed ball $B(0,R)$. If $p$ is continuous on $B(0,R)$, then then the maps $P_j : \mathcal{D}_R \to [0,1]$ are locally Lipschitz.
\end{lemma}

\begin{proof}
We first prove this in the setting where there are only two centers (i.e.\@ $k = 2$), before generalizing. Given two tuples of centers $w, w' \in \mathcal{D}_R$. Then the difference $P_j(w) - P_j(w')$ is:
\[P_j(w) - P_j(w') =  \int_{V_j(w) \setminus V_j(w')} p(x)\, dx - \int_{V_j(w') \setminus V_j(w)} p(x)\, dx.\]
Since $p$ is continuous on the closed set $B(0,R)$, it attains a maximum $p_\mathrm{max} = \sup\, p(x) < \infty$. Let $\lambda$ be the Lebesgue measure. It follows by triangle inequality that:
\[|P_j(w) - P_j(w')| \leq p_\mathrm{max}\cdot  \bigg(\lambda\big(\textcolor{gray}{V_j(w) \setminus V_j(w')}\big) + \lambda\big(\textcolor{gray}{V_j(w') \setminus V_j(w)}\big)\bigg).\]
Thus, to prove that $P_j$ is locally Lipschitz, we need to bound how much the $j$th Voronoi cell can grow/shrink when the two centers $w$ are perturbed to $w' \in \mathcal{D}_R$. As $w \in \mathcal{D}_R$, the two centers are separated $\|w_1 - w_2\| > 0$. We claim that if the perturbation $\|w - w'\|$ is a factor smaller than the separation, $\|w - w'\| \leq \frac{1}{4} \|w_1 - w_2\|$, then the $j$th Voronoi cell can only grow linearly with $\|w - w'\|$,
\[\lambda\big(\textcolor{gray}{V_j(w') \setminus V_j(w)}\big) \leq L_w \|w - w'\|,\]
for some $L_w > 0$. And, the same can be said for the other term, measuring how much the region can shrink. If this claim holds, then $P_j$ is locally Lipschitz, where the local Lipschitz constant at $w$ is $p_\mathrm{max}\cdot 2L_w$.

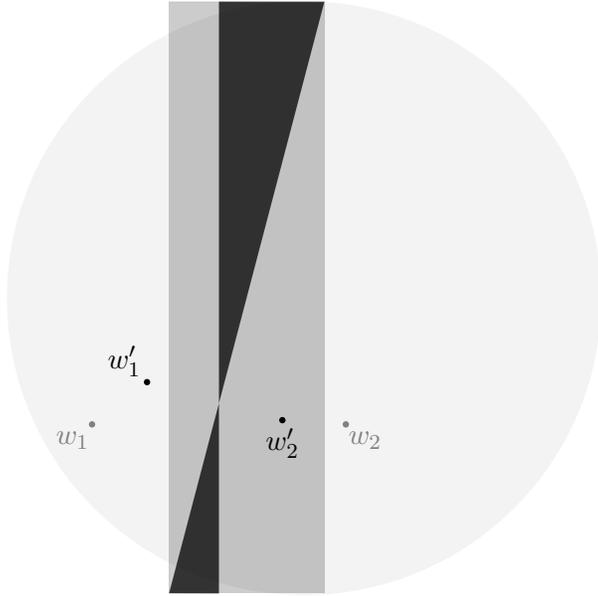
\begin{figure}
    \centering
    \def \globalscale {2.000000}
\begin{tikzpicture}[y=0.80pt, x=0.80pt, yscale=-\globalscale, xscale=\globalscale, inner sep=0pt, outer sep=0pt]
  \path[fill=black,fill opacity=0.05,even odd rule,line width=0.212pt] (70.1325,70.1323) ellipse (1.9756cm and 1.9756cm);

  \path[fill=black,fill opacity=0.2,line width=0.394pt,rounded corners=0.0000cm] (38.2029,0.0312) rectangle (74.9252,139.9688);

  \path[fill=black,fill opacity=0.75,line cap=butt,line join=miter,line width=0.212pt] (74.9252,0.0312) -- (38.2029,139.9688) -- (50.0000,140.0000) -- (50.0000,0.0000) -- cycle;

  \filldraw[gray] (20,100) circle (0.5pt) node[label={[label distance=2.5]250:$w_1$}]{};

  \filldraw[gray] (80,100) circle (0.5pt) node[label={[label distance=2.5]290:$w_2$}]{};

  \filldraw[black] (33,90) circle (0.5pt) node[label={[label distance=2.5]135:$w_1'$}]{};

  \filldraw[black] (65,99) circle (0.5pt) node[label={[label distance=2.5]270:$w_2'$}]{};
\end{tikzpicture}
    \captionsetup{width=0.9\textwidth}
    \caption{A two-dimensional projection of the 2-means problem in $\mathbb{R}^d$. The light gray disk represents the support of the distribution $p$, which has a diameter of $2R$. The initial tuple $w = (w_1,w_2)$ partitions the space along the vertical hyperplane. After a small perturbation to $w' = (w_1', w_2')$, a new Voronoi partition is induced, where the black region corresponds the symmetric difference $V_1(w)\Delta V_1(w') = V_2(w) \Delta V_2(w')$. The probability mass of this region can be upper bounded by the rectangular gray region whose width is $O(\|w - w'\|)$ and lengths in all other directions are $2R$.}
    \label{fig:lipschitz}
\end{figure}

Fix $w \in \mathcal{D}_R$ and let $2r = \|w_1 - w_2\|$ be the separation of its two centers. By a change of coordinates, we may without loss of generality assume that:
\[w_1 = ( -r, 0,\ldots, 0) \qquad \textrm{and}\qquad w_2 = (r,0\,\ldots,0).\]
Thus, the boundary of their Voronoi cells is the hyperplane $\{x \in \mathbb{R}^d : x_1 = 0\}$. We now show that if the perturbed centers $w'$ satisfy $\|w - w'\| = \epsilon \leq \frac{r}{2}$, then $V_1(w')$ is contained in the halfspace:
\[V_1(w') \subset \left\{x \in \mathbb{R}^d : x_1 \leq \left(1 + \frac{2R}{r} \right) \epsilon\right\},\] from which local Lipschitzness follows:
\[\lambda\big(\textcolor{gray}{V_1(w') \setminus V_1(w)}\big) \leq (2R)^{d-1}\left(1 + \frac{2R}{r}\right) \cdot \|w - w'\|,\]
since $V_1(w') \setminus V_1(w)$ is contained in the rectangular region where the last $d - 1$ coordinates have length $2R$ and the first coordinate length $(1 + 2R/r)\epsilon$. \Cref{fig:lipschitz} depicts this argument.

We show that $V_1(w')$ is contained in the above halfspace by upper bounding the first coordinate of points in $V_1(w')$. Note that the new boundary induced by $w'$ is the hyperplane $H$ intersecting $\frac{1}{2} (w_1' + w_2')$ defined by the normal vector $w_1' - w_2'$:
\[H:= \frac{1}{2} (w_1' + w_2')  + \big\{x \in \mathbb{R}^d : (w_1' - w_2')^\top x = 0\big\} .\]
Thus, $V_1(w')$ is to the left of $H$. The first term $\frac{1}{2} \|w_1' + w_2'\|$ contributes at most $\epsilon$ to the first coordinate of points in $H$, since $\|w - w'\| = \epsilon$. Since after the change of coordinates, all points in $B(0,R)$ must now be at most a distance of $2R$ away from $\frac{1}{2} (w_1' + w_2')$, we just need to bound the first coordinate of points in:
\[\big\{x \in B(0,2R) : (w_1' - w_2')^\top x = 0\big\}.\]
Let $w_1' - w_2' = (\alpha_1,\ldots, \alpha_d)$. Then if $x$ in this set satisfies:
\[|x_1| = \left| \frac{\alpha_2 x_2 + \cdots + \alpha_d x_d}{\alpha_1}\right| \leq \frac{\|w_1' - w_2'\| \cdot \|x\|}{r},\]
by Cauchy-Schwarz and the fact that $|\alpha_1| \geq r$, which follows  from the form of $w$ and that the perturbation is less than $r/2$. That is, $|x_1| \leq 2R\epsilon/r$. Thus, $V_1(w')$ is contained the above halfspace.

At this point, we have shown the result for $k = 2$. The setting for general $k$ is an easy extension. Let $\Delta$ be the symmetric difference. Then as before, we need to show:
\[\lambda \big(\textcolor{gray}{V_j(w) \Delta V_j(w')}\big) \leq 2L_w \|w - w'\|, \]
for some $L_w > 0$ and $w'$ in a neighborhood of $w$.

Given $w$ and fixed $j$, consider a collection of $k-1$ induced 2-means problems constructed on $\widetilde{w}_{\ell} := (w_j, w_{\ell})$ for $\ell \ne j$. Let $\widetilde{V}$ map the 2-center $\widetilde{w} \in \mathbb{R}^{2 \times d}$ to its Voronoi partitions. Then:
\[V_j(w) \Delta V_j(w') \subset \bigcup_{\ell \ne j} \widetilde{V}_j(\widetilde{w}_{\ell}) \Delta \widetilde{V}_j(\widetilde{w}_\ell').\]
It follows that we may reduce to the 2-center case, since:
\[\lambda\big(\textcolor{gray}{V_j(w) \Delta V_j(w')}\big) \leq \sum_{\ell \ne j} \lambda\big(\textcolor{gray}{\widetilde{V}_j(\widetilde{w}_{\ell}) \Delta \widetilde{V}_j(\widetilde{w}_\ell')}\big)\qedhere\]
\end{proof}

\subsection{Properties of $T(m)$}
Recall we defined for $r > 0$, the function $T_r : \mathbb{N} \to \mathbb{N}$ so that $T_r(m)$ is the unique natural number so that:
\begin{equation} 
    \sum_{m \leq n \textcolor{blue}{<} T_r(m)} \frac{1}{n} \leq r < \sum_{m \leq n \textcolor{orange}{\leq} T_r(m)} \frac{1}{n}. \tag{\ref{eqn:time-bound}}
\end{equation}
The following lemma and corollary give properties of $T_r$.

\begin{lemma}\label{lem:harmonic-partial-sum}
Let $1 < m < m'$ be in $\mathbb{N}$. Then:
\[\log \frac{m'}{m} \leq \sum_{m \leq n < m'} \frac{1}{n} \leq \log \frac{m' - 1}{m - 1}.\]
\end{lemma}
\begin{proof}
\[
   \log \frac{m'}{m} \leq  \int_m^{m'} \frac{1}{x} \, dx \leq \sum_{m \leq n < m'} \frac{1}{n} \leq \int_m^{m'} \frac{1}{x - 1}\, dx = \log \frac{m'-1}{m-1}.\qedhere
\]
\end{proof}

\vspace{10pt}

\begin{corollary}\label{cor:analytic-tr}
Let $r > 0$. Let $\alpha := e^r - 1$ and set $T \equiv T_r$. Then:
\[\alpha (m-1)  \leq T(m) - m \leq \alpha m.\]
\end{corollary}
\begin{proof}
Combining \Cref{lem:harmonic-partial-sum} with the definition of $T(m)$, we have:
\[\log \frac{T(m)}{m} \leq \sum_{m \leq n < T(m)} \frac{1}{n} \leq r < \sum_{m \leq n <T(m) + 1} \frac{1}{n}  \leq \log \frac{T(m)-1}{m-1}.\]
Rearranging yields the result.
\end{proof}

\section*{Acknowledgements}
Thanks to Yian Ma, Anthony Thomas, Zhi Wang, and AISTATS 2022 reviewers for helpful discussions that improved this work and its presentation.

\bibliography{references} 
\end{document}